\pdfoutput=1

\documentclass{article} 
\usepackage{iclr2025_conference,times}

\usepackage{amsmath,amsfonts,bm}

\def\eqref#1{equation~\ref{#1}}

\def\1{\bm{1}}

\DeclareMathAlphabet{\mathsfit}{\encodingdefault}{\sfdefault}{m}{sl}
\SetMathAlphabet{\mathsfit}{bold}{\encodingdefault}{\sfdefault}{bx}{n}

\usepackage{hyperref}
\usepackage{url}
\usepackage{algorithm}
\usepackage{algpseudocode}

\usepackage{booktabs}
\usepackage{graphicx}
\usepackage{xcolor}
\definecolor{seabornbg}{HTML}{F0F0F0}
\usepackage{enumitem}
\usepackage[most]{tcolorbox} 

\newtcolorbox{mybox}[1]{
  enhanced,
  breakable,
  colback=blue!5!white,
  colframe=blue!75!black,
  fonttitle=\mdseries,
  title=#1,
  boxrule=0.5pt,
  arc=3pt
}
\usepackage{xspace}

\usepackage{times}
\usepackage{latexsym}
\usepackage{cancel}

\usepackage[T1]{fontenc}

\usepackage[utf8]{inputenc}
\usepackage{listings}

\usepackage{microtype}
\usepackage{enumitem}

\usepackage{xcolor}
\definecolor{thedarkblue}{RGB}{0,0,120} 
\definecolor{mydarkblue}{rgb}{0,0.08,0.45} 
\definecolor{darkblue}{rgb}{0,0.08,180}
\colorlet{TufteRed}{red!80!black}
\definecolor{theblue}{RGB}{0,0,180}
\colorlet{thered}{TufteRed}
      
\usepackage{microtype}
\usepackage{balance}

\usepackage{booktabs}
\usepackage{tabularx}

\usepackage{amsmath,amssymb,amsthm}

\newcommand{\eat}[1]{\ignorespaces}
\usepackage{comment}

\newcommand{\journal}[1]{} 

\usepackage{tikz}
\usepackage{verbatim}
\usetikzlibrary{arrows}
\usetikzlibrary{shapes,snakes}
\usetikzlibrary{decorations.pathmorphing} 
\usetikzlibrary{fit}					
\usetikzlibrary{backgrounds}	

\usepackage{ragged2e}
\usepackage{multirow}
\usepackage{microtype}
\usepackage{balance}
\usepackage{setspace}

\graphicspath{{./}{./graphics/}}
\newcolumntype{H}{>{\setbox0=\hbox\bgroup}c<{\egroup}@{}}

\newcolumntype{R}[1]{>{\RaggedLeft\arraybackslash}} 
\newcolumntype{L}[1]{>{\RaggedRight\arraybackslash}} 

\AtBeginEnvironment{pmatrix}{\setlength{\arraycolsep}{2pt}}

\DeclareMathOperator{\hugeE}{\mbox{\huge\raise-0.3ex\hbox{E}}}
\DeclareMathOperator{\p}{\mathbb{P}}
\DeclareMathOperator{\hugep}{\mbox{\huge\raise-0.3ex\hbox{$\p$}}}

\usepackage[capitalize,noabbrev]{cleveref}

\usepackage{subcaption}
\usepackage{graphicx}
\usepackage{nicefrac}
\usepackage[most]{tcolorbox}

\DeclareMathAlphabet{\mathbcal}{OMS}{cmsy}{b}{n}
\usepackage{mathrsfs}
\usepackage{comment}
\usepackage{bm}
\usepackage{bbm}
\usepackage{amssymb}
\usepackage{enumitem}
\usepackage{paralist}
\usepackage{adjustbox}

\newtheorem{proposition}{Proposition}
\theoremstyle{remark}
\newtheorem{remark}{Remark}

\definecolor{googleblue}{HTML}{4285F4}
\definecolor{googlered}{HTML}{DB4437}
\definecolor{googlepurple}{HTML}{A142F4} 
\definecolor{googlegreen}{HTML}{0F9D58}

\usepackage[textsize=tiny,textwidth=1in]{todonotes}

\newcommand{\method}{\textsc{GraSPeR}\xspace}
\newcommand{\rebuttal}[1]{\textcolor{blue}{#1}}

\title{Reasoning-Based Personalized Generation for Users with Sparse Data}

\author{Bo Ni$^1$, Branislav Kveton$^2$, Samyadeep Basu$^2$, Subhojyoti Mukherjee$^2$, \\
\textbf{Leyao Wang$^3$, Franck Dernoncourt$^2$, Sungchul Kim$^2$, Seunghyun Yoon$^2$,} \\
\textbf{Zichao Wang$^2$, Ruiyi Zhang$^2$, Puneet Mathur$^2$, Jihyung Kil$^2$,} \\
\textbf{Jiuxiang Gu$^2$, Nedim Lipka$^2$, Yu Wang$^4$, Ryan A. Rossi$^2$, Tyler Derr$^1$} \\
\\
$^1$Vanderbilt University, $^2$Adobe Research, $^3$Yale University, $^4$University of Oregon \\
\texttt{\{bo.ni, tyler.derr\}@vanderbilt.edu} \\
\texttt{\{kveton, samyadeepb, subhomuk, dernonco, sukim, syoon, jackwa, ruizhang,} \\
\texttt{puneetm, jkil, jigu, lipka, ryrossi\}@adobe.com} \\
\texttt{leyao.wang.lw855@yale.edu, yuwang@oregon.com}
}

\lstdefinestyle{llmprompt}{
    basicstyle=\small\ttfamily,
    breaklines=true, 
    frame=single,
    framesep=5pt,
    backgroundcolor=\color{gray!10},
    numbers=left,
    numberstyle=\tiny\color{gray},
    numbersep=8pt,
}

\begin{document}
\iclrfinalcopy
\maketitle
\lhead{}

\begin{abstract}

Large Language Model (LLM) personalization holds great promise for tailoring
responses by leveraging personal context and history. However, real-world users usually possess sparse
interaction histories with limited personal context, such as cold-start users in
social platforms and newly registered customers in online E-commerce platforms,
compromising the LLM-based personalized generation. To address this challenge, we introduce \method (\textbf{Gra}ph-based \textbf{S}parse
\textbf{Pe}rsonalized \textbf{R}easoning), a novel framework for enhancing
personalized text generation under sparse context.
\method first augments user context by predicting items that the user would likely interact with in the future. With reasoning alignment, it then generates texts for these interactions to enrich the augmented context. In the end, it generates personalized outputs conditioned on both the real and synthetic histories, ensuring alignment with user style and preferences. Extensive experiments on three benchmark
personalized generation datasets show that \method achieves
significant performance gain, substantially improving
personalization in sparse user context settings.
\end{abstract}

\begin{figure*}[b!]
   \centering
    \scriptsize 
    \textcolor[HTML]{FF7F0E}{\rule{1.5ex}{1.5ex}} Ours (Graph+Reasoning)
\hspace{0.5em}
\textcolor[HTML]{1F77B4}{\rule{1.5ex}{1.5ex}} PGraph(Graph Only)
\hspace{0.5em}
\textcolor[HTML]{2CA02C}{\rule{1.5ex}{1.5ex}} LaMP(No Graph, No Reasoning) 
\hspace{0.5em}
\textcolor[HTML]{BCBD22}{\rule{1.5ex}{1.5ex}} REST-PG(Reasoning Only)
    \centering
    \begin{subfigure}[t]{0.49\textwidth}
        \centering
        \includegraphics[width=\linewidth]{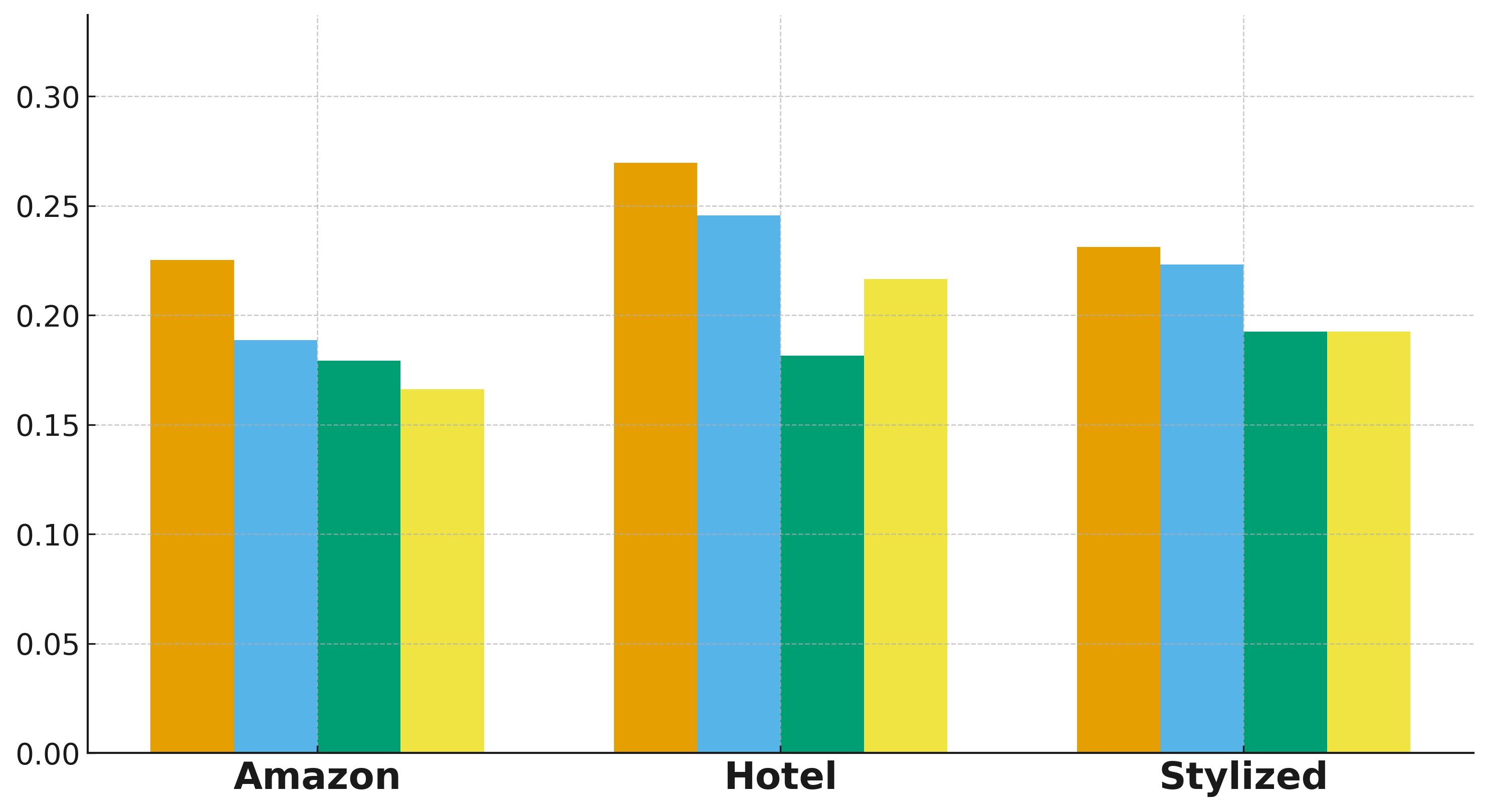}
        \caption{Long Text Generation.}
        \label{fig:long_text_generation}
    \end{subfigure}
    \hfill
    \begin{subfigure}[t]{0.49\textwidth}
        \centering
        \includegraphics[width=\linewidth]{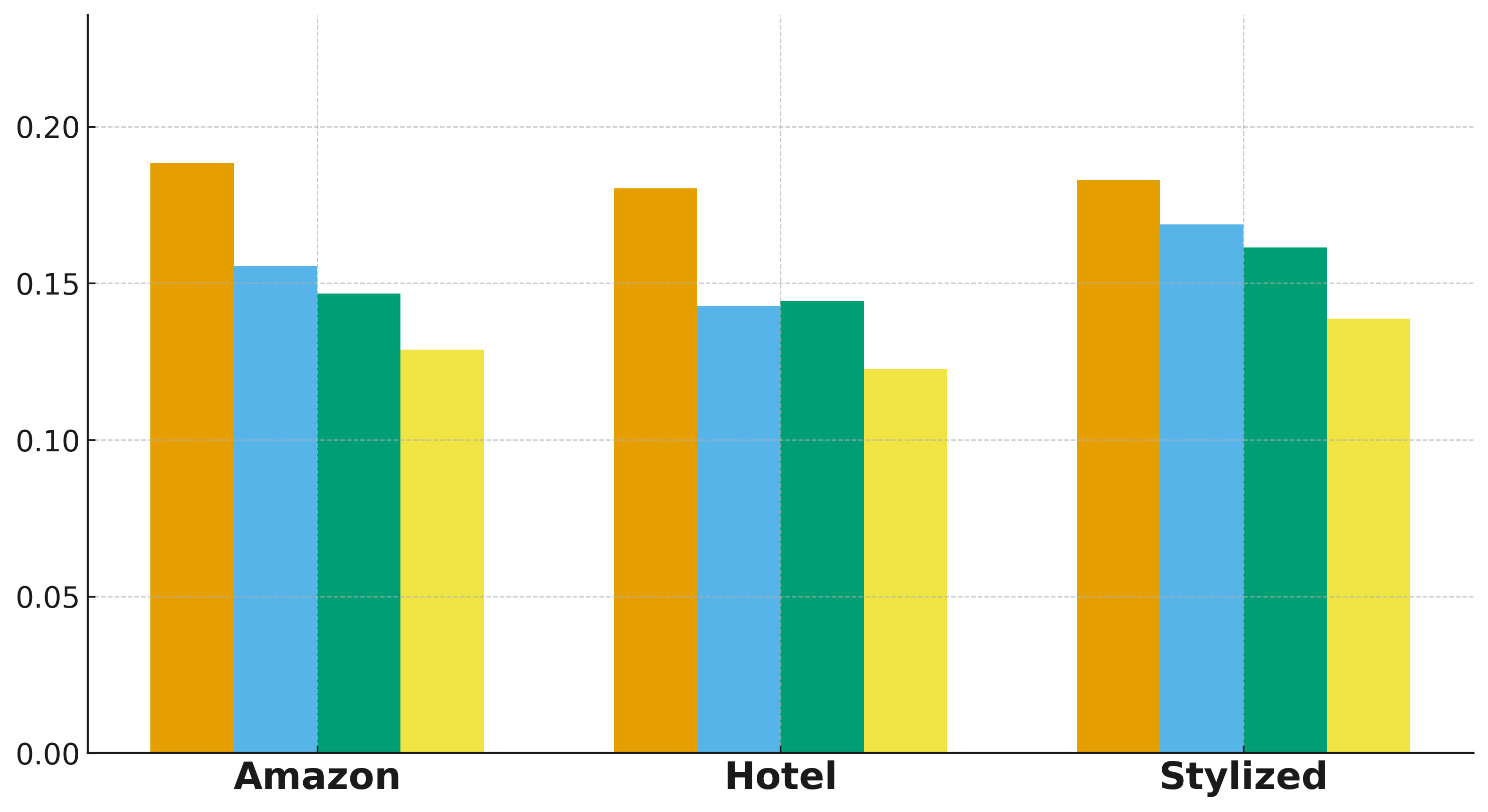}
        \caption{Short Text Generation.}
        \label{fig:short_text_generation}
    \end{subfigure}

    \caption{
    Results comparing our approach across two fundamental tasks and across datasets.
    Legends use descriptive labels: 
    Ours (Graph + Reasoning), PGraph (Graph Only), LaMP (No Graph, No Reasoning), and REST-PG (Reasoning Only). \method achieves over 10\% gains, on average, across datasets.}
    \label{fig:teaser}
\end{figure*}

\section{Introduction}
Personalized Large Language Models (LLMs) have recently garnered significant attention~\citep{salemi2023lamp, tsai-etal-2024-leveraging} due to their various downstream applications in search, recommendation, and conversational agents~\citep{Yoganarasimhan2019, 6648327, SHUMANOV2021106627}. By retrieving relevant personal context from user history, LLMs can produce outputs that are tailored to the given user's personal preferences and enhance overall satisfaction and quality.

The core of LLM personalization lies in retrieving personal context, typically
derived from a user’s history. However, most existing approaches emphasize
textual histories~\citep{salemi2023lamp, salemi2024optimization}. While useful,
these histories are often sparse and limited, which severely constrains
personalization for long-tail users due to insufficient context. For example, in
e-commerce platforms and social networks, more than 95\% 
of users tend to be cold-start~\citep{au_personalized_2025,ni-mcauley-2018-personalized}.
Recent studies show that incorporating auxiliary
product reviews from other users—naturally represented as user-item graphs that
connect users, items, and their interactions—can effectively improve model
performance in settings such as recommendation and text
generation~\citep{au_personalized_2025,wang2025bridgingreviewsparsityrecommendation}.
By leveraging such graph-based histories, language models are able to generate
more accurate and contextually rich personalized responses.

Despite these advancements, how to effectively leverage diverse data sources,
such as graphs, to improve data sparsity in personalized LLMs remains
underexplored. While prior work~\citep{au_personalized_2025} focused on
retrieving existing textual histories for target items, it ignores the potential
of more complex, structural information in the user-item interaction graphs.
Such rich structural data could provide crucial signals for users with otherwise
sparse textual histories, as shown in recent works \citep{wang2025bridgingreviewsparsityrecommendation}. Furthermore, the integration of varied contextual information
often lacks a dedicated reasoning phase before generation. We argue that this
reasoning step is crucial as the diversity and volume of retrieved information
requires the LLM to strategically synthesize and generate coherent and
personalized outputs~\citep{salemi_reasoning-enhanced_2025}.

To address these challenges, we propose \method, a framework with two key stages: 
graph-based augmentation and reasoning-aligned generation. In the augmentation stage, we integrate a pretrained link predictor to enrich the sparse user context by simulating potential future interactions. In the generation stage, the personalized text generation model leverages both the augmented history of the user and existing texts of the target item to craft tailored responses. Crucially, reasoning is explicitly incorporated into both stages: during augmentation, reasoning guides the generation of features for the simulated edges to ensure alignment with user preferences; during text generation, reasoning enforces consistency between the synthesized history and the final personalized response~\citep{salemi_reasoning-enhanced_2025}. Extensive experiments demonstrate that \method substantially outperforms state-of-the-art baselines.

Our contribution can be summarized as follows: 
\noindent

\textbf{(1)} We introduce \method, a novel framework that tackles sparsity in
personalized text generation by combining graph learning and reasoning.  

\textbf{(2)} To the best of our knowledge, this is the first work to explicitly
integrate reasoning into sparse personalization, enabling LLMs to cohesively
synthesize augmented context and generate faithful, user-aligned outputs.  

\textbf{(3)} We conduct extensive experiments on real-world datasets (Amazon~\citep{ni-mcauley-2018-personalized},
Hotel~\citep{kanouchi-etal-2020-may}, and Stylized Feedback~\citep{alhafni-etal-2024-personalized}),
showing that \method achieves over 10\% gains for text generation and 15\% for rating prediction.

\section{Problem Definition} \label{sec:problem_def}
Most existing works on personalized text generation assume that users have
access to rich personal contexts, typically requiring more than ten historical
text entries~\citep{au_personalized_2025, salemi2023lamp}. However, this
assumption does not align with real-world usage patterns, where user data often
follows a long-tail distribution—over 95\% of users have written fewer than two text entries~\citep{ni-mcauley-2018-personalized, au_personalized_2025}.
This data sparsity poses a significant challenge for personalization systems
that rely heavily on individual user histories.

We argue that the naturally occurring bipartite user-item interaction graph, constructed from collective user-item histories, offers a valuable source of auxiliary context~\citep{10.1145/3459637.3482313, Zhao_Liu_Neves_Woodford_Jiang_Shah_2021}. It encodes implicit relationships between users and items through shared interactions and textual feedback, providing a structural foundation to infer user preferences even in low-resource settings.

Thus, to address the sparsity challenge and enrich the user context in such cases, we follow the setting in prior work~\citep{au_personalized_2025} and define the task of \emph{graph-based personalized text generation}. In this task, the goal is to leverage both the sparse personal history of a user and the broader structural context encoded in the user-item graph to generate personalized text (e.g., reviews) for a target item.

\paragraph{Personalized Text Generation with Graph.}
We formally define the task in this section.  
Let $\mathcal{U} = \{u_1, u_2, \ldots, u_n\}$ denote the set of $n$ users, and $\mathcal{I} = \{i_1, i_2, \ldots, i_m\}$ denote the set of $m$ target items (e.g., products, hotels).  
For each user $u \in \mathcal{U}$, we define their interaction history as a sequence
\[
H_u = \big[ (i_{u,1}, t_{u,1}), \dots, (i_{u,k_u}, t_{u,k_u}) \big],
\]
where $k_u$ is the number of interactions of user $u$ (which may vary across users), $i_{u,\ell} \in \mathcal{I}$ is the $\ell$-th item interacted with by $u$, $t_{u,\ell} \in \mathcal{T}$ is the associated text, such as the review written by a customer or a comment written by a social media user, and $\mathcal{T}$ is the space of all possible text.

Following previous work~\citep{au_personalized_2025}, we model the interactions as a bipartite graph $G = (\mathcal{U} \cup \mathcal{I}, \mathcal{E})$, where each edge $(u, i) \in \mathcal{E}$ exists if $(i, t) \in H_u$ for some text $t$ and user $u$.

The goal of personalized text generation is to learn a function: 
\[
f: \mathcal{U} \times \mathcal{H} \times \mathcal{I} \rightarrow \mathcal{T}
\]
that generates a personalized text $t_{u,i^*}$ for a given user $u \in \mathcal{U}$, their history $H_u \in \mathcal{H}$, a new target item $i^* \in \mathcal{I}$ and the bipartite graph $G$:
\[
t_{u,i^*} = f(u, H_u, i^*),
\]
where $\mathcal{H}$ is the space of all possible user histories.

\begin{figure*}[t]
    \centering
    \includegraphics[width=\linewidth]{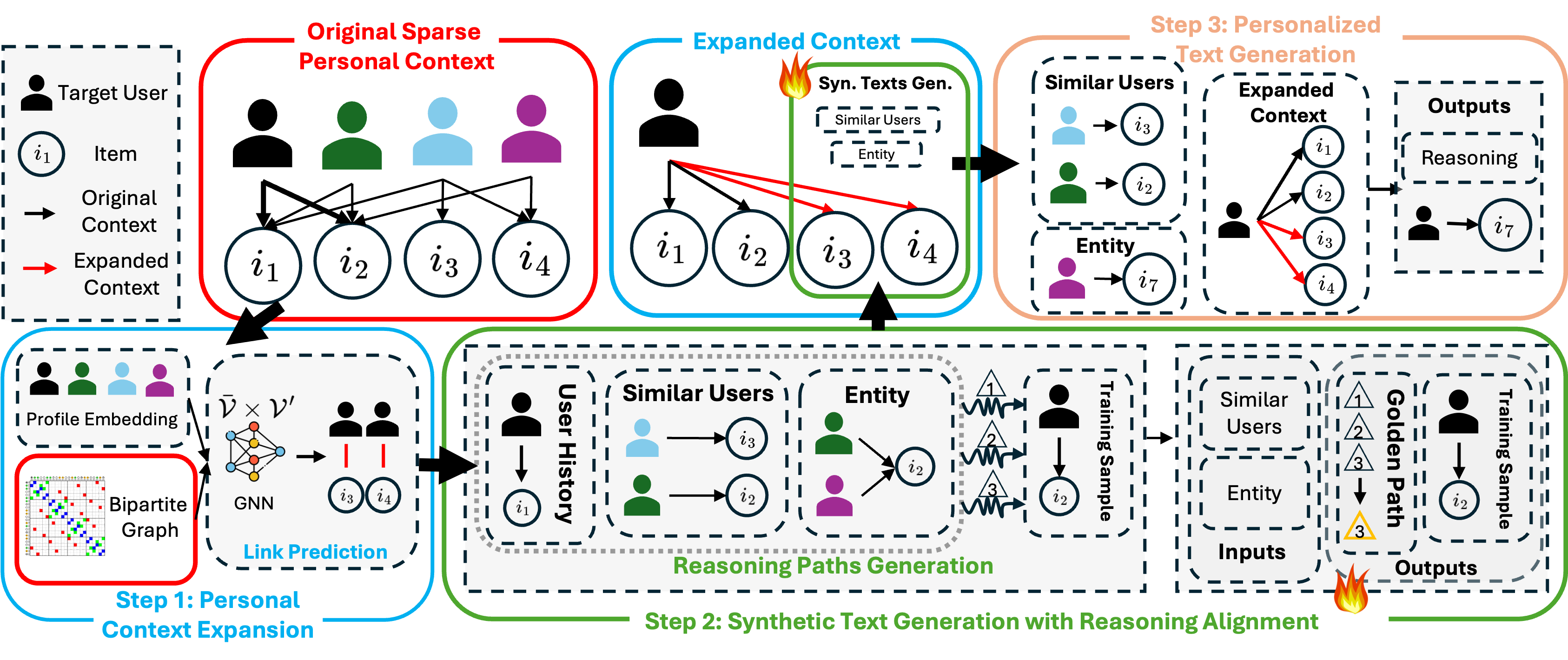}
    \caption{Overview of the proposed \method. 
    In \textbf{Step 1}, the personal context is enriched by leveraging the underlying graph structure to predict potential interactions. 
    In \textbf{Step 2}, we generate synthetic reviews for the predicted interactions with aligned reasoning. 
    In \textbf{Step 3}, the enhanced personal context that contains both the observed and simulated interactions enables the generation of more accurate and personalized text.}
    \label{fig:intro}
    \vspace{-0.1in}
\end{figure*}

\section{Method}
As illustrated in Figure~\ref{fig:intro}, \method operates in three steps to address the challenge of personalization under sparse context: (\textbf{Step 1}) \textit{Personal Context Expansion via Link Prediction}, which augments sparse histories with potential user--item interactions predicted from a graph model; (\textbf{Step 2}) \textit{Synthetic Review Generation with Reasoning Alignment}, which integrates intermediate reasoning paths into the text generation process to ensure personalization alignment; and (\textbf{Step 3}) \textit{Personalized Text Generation}, which produces personalized outputs conditioned on the augmented personal context. 

\subsection{Step 1: Personal Context Expansion}
Personalization often suffers in the long-tail setting where users have limited history, as limited contexts constrain downstream LLM generation quality and further compromise personalization. To mitigate this limitation, we propose enriching sparse histories by predicting additional user–item interactions through graph-based link prediction. These predicted edges allow us to construct a more comprehensive profile for each user, which can then be leveraged in later reasoning alignment and generation stages. To simulate additional user–item interactions,
we perform link prediction over the bipartite user–item interaction graph
$G=(\mathcal{U}\cup\mathcal{I},\mathcal{E})$ defined in Section~\ref{sec:problem_def}.

\paragraph{Training Graph-based Link Predictor.} Each node $v\in\mathcal{U}\cup\mathcal{I}$ is initialized with an embedding $\mathbf{h}_v^{(0)}\in\mathbb{R}^d$ derived from textual features using a pretrained encoder: 
\begin{equation}
    \mathbf{h}_v^{(0)} =
    \begin{cases}
        \text{Enc}\!\left(T_u\right), & v=u\in\mathcal{U},\\[6pt]
        \text{Enc}\!\left(R_i\right), & v=i\in\mathcal{I},
    \end{cases}
\end{equation}
where $\text{Enc}(\cdot)$ denotes a sentence encoder mapping text to a $d$-dimensional embedding, $T_u$ is the concatenated string of the sequence of text in user $u$'s profile $T_u=\text{Concat}[t_{u,1}, ..., t_{u,k_u}]$, and $R_i=\{\,t_{u,i}\mid (u,i)\in E\,\}$ is the set of texts aggregated for item $i$. $\text{Enc}$ can be text embedding model~\citep{reimers2019sentencebertsentenceembeddingsusing}. Since style is a crucial element for personalization compared to content itself~\citep{zhang2025personalizationlargelanguagemodels}, we instantiate Enc as a style-aware encoder that disentangles stylistic elements from content following~\cite{wegmann-etal-2022-author}.

With the initialized node features, we adopt GraphSAGE~\citep{hamilton2018inductiverepresentationlearninglarge} for inductive graph encoding. At each layer $l$, a node $v$ updates its representation by first aggregating its neighbors $\mathcal{N}(v)$ and then combining this aggregation with its own embedding:
\begin{equation}
    \mathbf{m}_v^{(l)} = \text{AGG}^{(l)}\!\left(\{\,\mathbf{h}_u^{(l-1)}: u\in \mathcal{N}(v)\,\}\right),\quad
    \mathbf{h}_v^{(l)}= \sigma\!\Big(\mathbf{W}^{(l)}\big[\mathbf{h}_v^{(l-1)} \parallel \mathbf{m}_v^{(l)}\big]\Big)
\end{equation}
where $\text{AGG}^{(l)}$ is a permutation-invariant mean aggregator, $\mathbf{W}^{(l)}$ is a learnable projection matrix, $\parallel$ denotes concatenation, and $\sigma$ is a non-linear activation function (ReLU in our implementation). After $L$ layers, we obtain the final node embedding $\mathbf{z}_v=\mathbf{h}_v^{(L)}$. In conjunction with the style-aware initialization, $\mathbf{z}_v$ captures both preference and stylistic elements of the personalization target.

To estimate the likelihood of a new interaction between a user $u$ and an item $i$, we apply a decoder over their embeddings. Specifically, we compute a score with a multi-layer perception (MLP) and turn it into a probability as:
\begin{equation}
    s_{u,i^*}=\text{MLP}\!\left([\mathbf{z}_u \parallel \mathbf{z}_{i^*}]\right),\qquad
    \hat{y}_{u,i^*}=\text{Sigmoid}\!\left(s_{u,i^*}\right).
    \label{eq:graph inference}
\end{equation} 
The link predictor is trained with binary cross-entropy (BCE) loss, where observed interactions are positives and uniformly sampled non-interacted pairs are negatives:
\begin{equation}
\mathcal{L}_{\text{link}}
= -\!\!\sum_{(u,i)\in E^{+}}\!\!\log \sigma(s_{u,i})
 \;-\!\!\sum_{(u,i)\in E^{-}}\!\!\log\!\big(1-\sigma(s_{u,i})\big),
\end{equation}
with $E^{+}=E$ and $E^{-}$ constructed via negative sampling~\citep{zhang2018linkpredictionbasedgraph}.

\paragraph{Inference and Profile Augmentation.}
At inference, for each sparse user $u$, we score all candidate items 
$i \in \mathcal{I}\setminus\{\,i:(u,i)\in \mathcal{E}\,\}$, rank them by 
$s_{u, i^*}$, and select the top-$K$ predictions. Let 
$\mathcal{I}_u^K = \{\,i_{u,1},\dots,i_{u,K}\,\}$ denote this top-$K$ set. 
For each $i \in \mathcal{I}_u^K$, we generate a synthetic text (as detailed in Step~2) to approximate how $u$ 
might interact with it, and these texts are appended to $u$’s 
history. The result is an augmented user profile that incorporates both 
observed and predicted interactions:
\begin{equation}
\label{eq:tildeH}
   \tilde{H}_u = H_u \cup \{\,\tilde{t}_{u,i_{u,k}} : k=1,\dots,K\,\}.
\end{equation} 
where $\tilde{t}_{u,i}$ explicitly denotes the synthetically generated text 
for user $u$ on item $i$. Importantly, these predicted edges are used only locally for the given user $u$. 
For example, an edge $(u,i)$ generated for user $u$ does not affect another 
user $u'$ through the shared item $i$. Thus, we are not reconstructing a 
full new bipartite graph but enriching each user profile independently 
for downstream personalization. 

\subsection{Step 2: Synthetic Text Generation with Reasoning Alignment}
The augmented profiles from Step 1 yield predicted candidate items, but directly generating texts from these signals risks propagating noise or stylistic mismatch. To address this, we design a reasoning-based synthetic text generation process that integrates explicit reasoning before producing final outputs.

\paragraph{Synthetic Text Generation Setup.}  
For a target user $u$, let $\mathcal{S}_u$ denote the set of similar users identified from the user–item graph $G$, and let $\mathcal{I}_u$ be the set of items inferred in Step~1 that $u$ is likely to interact with. $\mathcal{S}_u$ is obtained by calculating the cosine similarity between the node embeddings $\mathbf{z}_v$. We select the top 3 similar users to construct $H_{S_u}$. For each item $i \in \mathcal{I}_u$, we aim to generate a synthetic text $\tilde{y}_{u,i}$ that reflects $u$’s style and preferences. The generation process conditions on three sources of input:
\begin{equation}
\label{eq:input}
    x = \{ H_u, \, H_{\mathcal{S}_u}, \, P_{u,i} \},
\end{equation}
where $H_u$ are past texts of $u$, $H_{\mathcal{S}_u}$ are texts written by similar users, and $P_{u,i}$ are peer texts associated with item $i$. Following \citet{au_personalized_2025}, $P_{u,i}$ is constructed by ranking all texts associated with item $i$ using the BM25 retrieval model~\citep{robertson1995okapi} and selecting the top 4 most relevant entries, where relevance is measured by semantic similarity to the input query. 

\paragraph{Reasoning Path Generation.}  
We use a language model $\mathcal{M}$ to produce intermediate rationales that explain why $u$ might write for item $i$ in a certain way. Formally, a reasoning path $\mathcal{Z}$ is an intermediate textual explanation conditioned on $x$, such as \emph{``User $u$ tends to prefer lightweight laptops; similar users highlighted battery life for item $i$; hence the review should emphasize portability and battery.''} During training, for each text entry $t_{u,j}$, we let $x=\{H_u\setminus t_{u,j}, H_{\mathcal{S}_u}, P_{u,j}\}$. We then obtain a set of candidate reasoning paths by sampling with a formatted prompt $\phi(x, t_{u,j})$(The prompt formulation is supplied in Appendix~\ref{appendices:prompts}) that includes the input and the expected output
\begin{equation}
\label{eq:reasoning_sample}
    \mathcal{Z}^{(r)} = \mathcal{M}(\phi(x, t_{u,j})), \quad r=1,2,\ldots,R.
\end{equation}
We sample the candidate reasoning paths because not all candidate reasoning paths are equally reliable. We select a \emph{golden} reasoning path $\mathcal{Z}^*$ that best aligns with ground-truth outputs by maximizing task performance under an evaluation metric $\Omega$, which we take the average of the ROUGE and METEOR scores (see \cref{app:metrics} for more details on the evaluation metrics)
\begin{equation}
\label{eq:reasoning_gold}
\mathcal{Z}^* = \arg\max_{\mathcal{Z}} \, \Omega\big(t'_{u,j}, t_{u,j}\big), \quad
\text{where } t'^{(r)}_{u,j} = \mathcal{M}(\xi(x, \mathcal{Z}^{(r)})),
\end{equation}
with $t'^{(k)}_{u,j}$ denoting the generated synthetic text, and $\xi$ a prompt-construction function combining $x$ and $\mathcal{Z}$. The specification of $\xi$ is supplied in Appendix~\ref{appendices:prompts}.

\paragraph{Reasoning Alignment.}  
Finally, the model is fine-tuned to jointly generate the selected reasoning path $\mathcal{Z}^*$ and text $t_{u,j}$:
\begin{equation}
\label{reasoning_final}
    \mathcal{L}_{\text{gen}} = \text{CE}\!\left(\mathcal{M}(\rho(x)), \, t_{u,j}\right),
\end{equation} 
where $\rho$ is the prompt formatting function that is designed to generate both the reasoning and the output (The prompt formulation is supplied in Appendix~\ref{appendices:prompts}). $\mathcal{M}(\rho(x))$ is trained to output $\mathcal{Z}^*$ followed by the text. It enables the model to leverage noisy augmentation while staying faithful to user-specific style and preferences. Let the fine-tuned model be $\mathcal{M}'$, for a predicted item $i\in\mathcal{I}_u^{K}$, the final synthetic text $\tilde{t}_{u,i}$ is then generated with $\mathcal{M}'(x)\setminus \mathcal{Z}'$ where $\mathcal{Z}'$ is the generated reasoning.

\subsection{Step 3: Personalized Text Generation}
\label{sec:text_gen}

The final stage reuses the reasoning aligned language model from Step~2, which defines a mapping from user histories, similar-user signals, and candidate items to reasoning paths and synthetic texts. In Step~3, we employ the same function for the personalization task.

Let $\mathcal{I}_u$ be the set of candidate items inferred in Step~1. For each $i \in \mathcal{I}_u$, Step~2 already learns to map from the user's profile history, texts written by similar users, and peer texts associated with the predicted target item, as shown in ~\cref{eq:input}. Note that for personalized text generation for a specific target item $i^*$, we have the same input formulation $x^*=\big[\tilde{H}_u, \, H_{S_u}, P_{u, i^{*}}\big]$ where insteadwof profile history, $\tilde{H}_u$ is the augmented history obtained from Step~1. Since the reasoning–generation function is shared, we reuse the fine-tuned model $\mathcal{M}'$ in Step~2 for the target text generation. Given the input $x^*$, the fine-tuned model $\mathcal{M}'$ produces both a reasoning path $z^*$ and a personalized text $\hat{t}_{u,i^*}$:
\begin{equation}
\mathcal{M}'(x^*) = [z^* \parallel \hat{t}_{u,i^*}].
\end{equation}
The final prediction strips away the reasoning tokens:
\begin{equation}
\hat{t}_{u,i^*} = \mathcal{M}'(x^*) \setminus z^*.
\end{equation}
\section{Experiments}
\label{sec:experiment}

\subsection{Datasets and Metrics}
We follow the experiment setup in prior works~\citep{au_personalized_2025}, which consists of three datasets: Amazon Review~\citep{ni-mcauley-2018-personalized}, Hotel Review~\citep{kanouchi-etal-2020-may}, and Stylized Feedbacks~\citep{alhafni-etal-2024-personalized}. We cover three personalization tasks: long text generation, short text generation, and rating prediction. For long text generation, a title will be given to guide the generation, and for short text generation, a paragraph will be given for summarization. 

For both long and short text generation tasks, we adopt widely used lexical overlap metrics, including ROUGE-1, ROUGE-L, and METEOR, following prior work~\citep{au_personalized_2025, kumar2024longlampbenchmarkpersonalizedlongform, salemi2023lamp}. To complement these surface-level metrics, we further incorporate LLM-as-a-Judge evaluation~\citep{salemi_reasoning-enhanced_2025, liu-etal-2023-g}, where a strong language model provides comparative assessments of personalization. Additional details on LLM-as-a-Judge 
can be found in \cref{appendix:lj}. For rating prediction, we use Root Mean Squared Error (RMSE) and Mean Average Error (MAE). 
Further details on the metrics and experimental setup can be found in \cref{appendices:exp-set-up}.

\begin{table*}[t]
\renewcommand{\arraystretch}{0.85}
    \centering
    \scriptsize
    \caption{
    Results comparing the proposed approach called \method to state-of-the-art methods across 3 different tasks  
    on the Amazon Review benchmark.}
    
    \label{tab:amazon_results}
    \begin{tabular}{@{}llc@{}cccc@{}}
        \toprule
        \textbf{Task} & \textbf{Metric} & \textbf{LLM} & 
        \textbf{\method {\small \rm \;(ours)}} & 
        \textbf{PGraph} & \textbf{LaMP} & \textbf{REST-PG} \\
        \midrule
        \multirow{9}{*}{\textcolor{googlegreen}{\textbf{\sc \bfseries Long Text Gen.}}}
        & \multirow{2}{*}{ROUGE-1 $\uparrow$}  & \textit{4o-mini} & \textbf{0.219} & 0.189  & 0.171 &  N/A \\
        &                           & \textit{LlaMA3}     & \textbf{0.215} & 0.178  & 0.173 & 0.165 \\
        \cmidrule(l){2-7}
        & \multirow{2}{*}{ROUGE-L $\uparrow$}  & \textit{4o-mini}             & \textbf{0.170} & 0.130  & 0.117 & N/A \\

        &              & \textit{LlaMA3}               & \textbf{0.171} & 0.129  & 0.129 & 0.109 \\
        \cmidrule(l){2-7}
        & \multirow{2}{*}{METEOR $\uparrow$}  & \textit{4o-mini}             & 0.182 & \textbf{0.196}  & 0.176 & N/A \\
        &                           & \textit{LlaMA3}               & \textbf{0.178} & 0.151  & 0.138 & 0.122 \\
        \cmidrule(l){2-7}
        & \multirow{2}{*}{LLM-as-a-Judge $\uparrow$}  & \textit{4o-mini}  & \textbf{0.421}  & 0.389 & 0.328 & N/A \\
        &              & \textit{LlaMA3} & \textbf{0.337} & 0.297  & 0.277 & 0.269 \\
        \midrule
        \multirow{9}{*}{\textcolor{googleblue}{\textbf{\sc \bfseries Short Text Gen.}}}
        & \multirow{2}{*}{ROUGE-1 $\uparrow$}  & \textit{4o-mini}             & \textbf{0.178} & 0.115  & 0.108 & N/A \\
        &                           & \textit{LlaMA}               & \textbf{0.155} & 0.131  & 0.124 & 0.089 \\
        \cmidrule(l){2-7}
        & \multirow{2}{*}{ROUGE-L $\uparrow$} & \textit{4o-mini}             & \textbf{0.174} & 0.112  & 0.105 & N/A \\
        &                           & \textit{LlaMA3}               & \textbf{0.153} & 0.125  & 0.118 & 0.081 \\
        \cmidrule(l){2-7}
        & \multirow{2}{*}{METEOR $\uparrow$} & \textit{4o-mini}             & \textbf{0.162} & 0.099  & 0.091 & N/A \\
        &                           & \textit{LlaMA3}               & \textbf{0.142} & 0.125  & 0.117 & 0.113 \\
        \cmidrule(l){2-7}
        & \multirow{2}{*}{LLM-as-a-Judge $\uparrow$}   & \textit{4o-mini}  & \textbf{0.406} & 0.353 & 0.334 & N/A \\
        &              & \textit{LlaMA3}               & \textbf{0.304} & 0.241  & 0.228 & 0.232 \\
        \midrule
        \multirow{4}{*}{\textcolor{googlered}{\textbf{\sc \bfseries Rating Pred.}}}
        & \multirow{2}{*}{RMSE $\downarrow$}  & \textit{4o-mini}             & \textbf{0.33}  & 0.38   & 0.34  & N/A \\
        \multirow{4}{*}{\textcolor{googlered}{\scriptsize (Recommendation)}}
        &                           & \textit{LlaMA3}               & \textbf{0.32}  & 0.76   & 0.72  & 0.65 \\
        \cmidrule(l){2-7}
        & \multirow{2}{*}{MAE $\downarrow$}    & \textit{4o-mini}             & \textbf{0.34}  & 0.73   & 0.70  & N/A \\
        &                           & \textit{LlaMA3}               & \textbf{0.31}  & 0.34   & 0.31  & 0.46 \\
        \bottomrule
    \end{tabular}
\end{table*}
\begin{table*}[t!]
\renewcommand{\arraystretch}{0.75}
    \centering
    \scriptsize
    \caption{Performance Metrics for Hotel Experience Generation.}
    \label{tab:hotel_results}
    \begin{tabular}{@{}llc@{}cccc@{}}
        \toprule
        \textbf{Task} & \textbf{Metric} & \textbf{LLM} &
        \textbf{\method {\small \rm \;(ours)}} &
        \textbf{PGraph} & \textbf{LaMP} & \textbf{REST-PG} \\
        \midrule
        \multirow{8}{*}{\textcolor{googlegreen}{\textbf{\sc \bfseries Long Text Gen.}}}
        & \multirow{2}{*}{ROUGE-1 $\uparrow$}  & \textit{4o-mini} & 0.257 & \textbf{0.263} & 0.221 & N/A \\
        &                                     & \textit{LlaMA3}  & 0.258 & \textbf{0.263} & 0.199 & 0.221 \\
        \cmidrule(l){2-7}
        & \multirow{2}{*}{ROUGE-L $\uparrow$} & \textit{4o-mini} & \textbf{0.163} & 0.152 & 0.135 & N/A \\
        &                                     & \textit{LlaMA3}  & \textbf{0.168} & 0.157 & 0.129 & 0.131 \\
        \cmidrule(l){2-7}
        & \multirow{2}{*}{METEOR $\uparrow$}  & \textit{4o-mini} & 0.165 & 0.184 & 0.164 & N/A \\
        &                                     & \textit{LlaMA3}  & 0.165 & \textbf{0.191} & 0.152 & 0.145 \\
        \cmidrule(l){2-7}
        & \multirow{2}{*}{LLM-as-a-Judge $\uparrow$} & \textit{4o-mini} & \textbf{0.500} & 0.414 & 0.300 & N/A \\
        &                                            & \textit{LlaMA3}  & \textbf{0.488} & 0.372 & 0.246 & 0.369 \\
        \midrule
        \multirow{8}{*}{\textcolor{googleblue}{\textbf{\sc \bfseries Short Text Gen.}}}
        & \multirow{2}{*}{ROUGE-1 $\uparrow$} & \textit{4o-mini} & \textbf{0.135} & 0.112 & 0.108 & N/A \\
        &                                     & \textit{LlaMA3}  & \textbf{0.147} & 0.127 & 0.126 & 0.100 \\
        \cmidrule(l){2-7}
        & \multirow{2}{*}{ROUGE-L $\uparrow$} & \textit{4o-mini} & \textbf{0.128} & 0.111 & 0.104 & N/A \\
        &                                     & \textit{LlaMA3}  & \textbf{0.140} & 0.118 & 0.117 & 0.091 \\
        \cmidrule(l){2-7}
        & \multirow{2}{*}{METEOR $\uparrow$}  & \textit{4o-mini} & \textbf{0.120} & 0.081 & 0.075 & N/A \\
        &                                     & \textit{LlaMA3}  & \textbf{0.130} & 0.102 & 0.106 & 0.084 \\
        \cmidrule(l){2-7}
        & \multirow{2}{*}{LLM-as-a-Judge $\uparrow$} & \textit{4o-mini} & \textbf{0.469} & 0.360 & 0.346 & N/A \\
        &                                            & \textit{LlaMA3}  & \textbf{0.304} & 0.224 & 0.228 & 0.215 \\
        \midrule
        \multirow{4}{*}{\textcolor{googlered}{\textbf{\sc \bfseries Rating Pred.}}}
        & \multirow{2}{*}{RMSE $\downarrow$}  & \textit{4o-mini} & 0.660 & \textbf{0.328} & 0.340 & N/A \\
        &                                     & \textit{LlaMA3}  & \textbf{0.322} & 0.347 & 0.326 & 0.335 \\
        \cmidrule(l){2-7}
        & \multirow{2}{*}{MAE $\downarrow$}   & \textit{4o-mini} & 0.356 & \textbf{0.336} & 0.700 & N/A \\
        &                                     & \textit{LlaMA3}  & \textbf{0.520} & 0.724 & 0.680 & 0.642 \\
        \bottomrule
    \end{tabular}
    \vspace{-0.5em}
\end{table*}

\subsection{Baselines}
We benchmark against three state-of-the-art personalization baselines. \textbf{LaMP}~\citep{salemi2023lamp} conditions on a user's past writing via prompts but uses no graph learning or reasoning. \textbf{PGraphRAG}~\citep{au_personalized_2025} (i.e., PGraph) employs graph-based retrieval augmented generation with BM25 for personalization but lacks reasoning or fine-tuning. \textbf{REST-PG}~\citep{salemi_reasoning-enhanced_2025} models user preferences through reasoning paths with iterative fine-tuning, but it does not employ any context expansion. Expanded baseline descriptions can be found in \cref{appendices:exp-set-up}. We evaluate LaMP and PGraph with LLaMA-3-8b-instruct and GPT-4o mini; REST-PG is implemented with LLaMA-3-8b-instruct only due to its fine-tuning requirement. 

\subsection{Main Results}
We report the experimental results for the Amazon Reviews, Hotel Experience, and Stylized Feedback datasets in \cref{tab:amazon_results,tab:hotel_results,tab:gap_results}, respectively. These results evaluate our approach across three tasks—long text generation, short text generation, and ordinal classification—and compare its performance to the state-of-the-art personalization baselines. Additional results are provided in Appendix ~\ref{appendices:additional_results}.

Overall, our method consistently outperforms the baselines across the datasets and tasks. For long text generation, we observe significant improvements in ROUGE-1, ROUGE-L, and METEOR scores, demonstrating the model’s ability to generate more accurate and contextually relevant outputs through reasoning enhanced retrieval and generation. Notably, the Amazon Reviews dataset shows the largest performance gains. This can be attributed to it having the fewest average degree of 1.68 comparing to 2.12 and 2.42 for the Hotel Experience dataset and Sylized Feedback dataset (see Appendix~\ref{appendices:dataset}). The increased sparsity leads to better performance gain with the proposed method. For ordinal classification, our approach achieves lower RMSE and MAE compared to the baselines, indicating better alignment with user rating tendencies. 
\begin{table*}[t!]
\renewcommand{\arraystretch}{0.75}
    \centering
    \scriptsize
    \caption{Performance Metrics for Stylized Feedback Generation.}
    \label{tab:gap_results}
    \begin{tabular}{@{}llc@{}cccc@{}}
        \toprule
        \textbf{Task} & \textbf{Metric} & \textbf{LLM} &
        \textbf{\method {\small \rm \;(ours)}} &
        \textbf{PGraph} & \textbf{LaMP} & \textbf{REST-PG} \\
        \midrule
        \multirow{8}{*}{\textcolor{googlegreen}{\textbf{\sc \bfseries Long Text Gen.}}}
        & \multirow{2}{*}{ROUGE-1 $\uparrow$}  & \textit{4o-mini} & \textbf{0.214} & 0.185 & 0.187 & N/A \\
        &                                     & \textit{LlaMA3}  & \textbf{0.235} & 0.217 & 0.186 & 0.189 \\
        \cmidrule(l){2-7}
        & \multirow{2}{*}{ROUGE-L $\uparrow$} & \textit{4o-mini} & \textbf{0.147} & 0.123 & 0.123 & N/A \\
        &                                     & \textit{LlaMA3}  & \textbf{0.175} & 0.158 & 0.134 & 0.127 \\
        \cmidrule(l){2-7}
        & \multirow{2}{*}{METEOR $\uparrow$}  & \textit{4o-mini} & 0.178 & \textbf{0.183} & 0.189 & N/A \\
        &                                     & \textit{LlaMA3}  & 0.175 & \textbf{0.178} & 0.177 & 0.173 \\
        \cmidrule(l){2-7}
        & \multirow{2}{*}{LLM-as-a-Judge $\uparrow$} & \textit{4o-mini} & \textbf{0.423} & 0.399 & 0.318 & N/A \\
        &                                           & \textit{LlaMA3}  & \textbf{0.340} & 0.340 & 0.273 & 0.281 \\
        \midrule
        \multirow{8}{*}{\textcolor{googleblue}{\textbf{\sc \bfseries Short Text Gen.}}}
        & \multirow{2}{*}{ROUGE-1 $\uparrow$}  & \textit{4o-mini} & \textbf{0.137} & 0.122 & 0.113 & N/A \\
        &                                     & \textit{LlaMA3}  & \textbf{0.160} & 0.149 & 0.140 & 0.097 \\
        \cmidrule(l){2-7}
        & \multirow{2}{*}{ROUGE-L $\uparrow$} & \textit{4o-mini} & \textbf{0.133} & 0.118 & 0.109 & N/A \\
        &                                     & \textit{LlaMA3}  & \textbf{0.157} & 0.142 & 0.134 & 0.091 \\
        \cmidrule(l){2-7}
        & \multirow{2}{*}{METEOR $\uparrow$}  & \textit{4o-mini} & \textbf{0.144} & 0.104 & 0.096 & N/A \\
        &                                     & \textit{LlaMA3}  & 0.131 & \textbf{0.142} & 0.136 & 0.112 \\
        \cmidrule(l){2-7}
        & \multirow{2}{*}{LLM-as-a-Judge $\uparrow$} & \textit{4o-mini} & \textbf{0.395} & 0.343 & 0.331 & N/A \\
        &                                           & \textit{LlaMA3}  & \textbf{0.284} & 0.242 & 0.236 & 0.255 \\
        \midrule
        \multirow{4}{*}{\textcolor{googlered}{\textbf{\sc \bfseries Rating Pred.}}}
        & \multirow{2}{*}{RMSE $\downarrow$}  & \textit{4o-mini} & \textbf{0.637} & 0.673 & 0.667 & N/A \\
        &                                     & \textit{LlaMA3}  & \textbf{0.684} & 0.724 & 0.680 & 0.678 \\
        \cmidrule(l){2-7}
        & \multirow{2}{*}{MAE $\downarrow$}   & \textit{4o-mini} & \textbf{0.332} & 0.347 & 0.344 & N/A \\
        &                                     & \textit{LlaMA3}  & \textbf{0.337} & 0.347 & 0.327 & 0.326 \\
        \bottomrule
    \end{tabular}
    \vspace{-0.5em}
\end{table*}

Furthermore, we find that the relative advantage of our method becomes more pronounced under LLM-as-a-Judge evaluation. Compared with conventional textual similarity metrics, which primarily capture surface-level overlap, LLM-as-a-Judge better aligns with human preference when assessing personalization~\citep{salemi_reasoning-enhanced_2025, liu-etal-2023-g}. This is because personalization often extends beyond literal similarity to reflect nuanced aspects such as style, tone, and contextual coherence. The gains observed in this setting demonstrate that \method, by jointly leveraging context expansion and reasoning alignment, is able to generalize user-specific stylistic patterns more effectively, leading to outputs that are not only accurate but also more faithful to individual user preferences.

\begin{table*}[t]
\renewcommand{\arraystretch}{0.75}
    \centering
    \scriptsize
    \caption{Module ablation studies for Amazon Reviews dataset with the Llama3 backbone.
    \method-ft is an ablation without fine-tuning, and \method-r-ft is without reasoning and fine-tuning. 
    }
    \label{tab:ablation}
    \begin{tabular}{llcccccc}
        \toprule
        \textbf{Task} & \textbf{Metric} & \method & \method-ft & \method-r-ft & \textbf{PGraph} & \textbf{LaMP} & \textbf{REST-PG}\\
        \midrule
        \textbf{Text Generation} & ROUGE-1 & \textbf{0.215} & 0.175 & \underline{0.182}  & 0.178  & 0.173 & 0.165 \\
        & ROUGE-L & \textbf{0.171} & 0.121 & 0.125 & \underline{0.129} & \underline{0.129} & 0.009 \\
        & METEOR & 0.178 &  \underline{0.188} & \textbf{0.200} & 0.151 & 0.138 & 0.122 \\
        \midrule
        \textbf{Title Generation} & ROUGE-1 & \textbf{0.155} & 0.100 & 0.118 & \underline{0.131} & 0.124 & 0.112 \\
        & ROUGE-L & \textbf{0.153} & 0.098 & 0.117 & \underline{0.125} & 0.118 & 0.077\\
        & METEOR & \textbf{0.142} & 0.085 & 0.111 & \underline{0.125} & 0.117 & 0.113 \\
        \midrule
        \textbf{Rating Prediction} & RMSE & \textbf{0.32} & 0.53 & \underline{0.52} & \underline{0.52} & 0.72 & 0.65 \\
        & MAE & \textbf{0.31} & 0.39 & 0.36 & \underline{0.34} & \textbf{0.31} & 0.46 \\
        \bottomrule
    \end{tabular}
\end{table*}

\subsection{Ablation Studies}
We conduct extensive ablations on \method, examining its components to validate the framework's effectiveness. We further analyze the key hyperparameter K (number of predicted items) along with supporting theory. Results using additional language models are presented in Appendix~\ref{appendices:additional_results}. 

\subsubsection{Model Variants}
In Table~\ref{tab:ablation}, we ablate \method with additional variants to demonstrate the effectiveness of its two main contributions: personal context expansion and reasoning alignment. \method-ft removes the fine-tuning for reasoning, which means it will only include the reasoning prompt. \method-r-ft further removes the reasoning process and lets the model directly generate the final output with the input as specified in \cref{eq:input}. The full model (\method) consistently outperforms its reduced counterparts, 
confirming that each component is indispensable to the framework. 

\noindent \textbf{Effect of Context Expansion.} 
Personal context expansion provides additional evidence for personalization, but without proper reasoning, the augmented context can introduce noise. This is reflected in \method-ft-r, which relies solely on link prediction without reasoning or fine-tuning. These results indicate that context expansion alone is insufficient and may even hurt performance if not paired with reasoning alignment. The effect of noise-induced bias is extensively discussed in \cref{appendices:theory}. 

\noindent\textbf{Effect of Reasoning Alignment.}
Reasoning alignment ensures that the augmented context contributes in a way that matches user preferences and task requirements. Comparing \method-ft (with reasoning but no fine-tuning) to \method shows that reasoning alignment improves performance across metrics. Additionally, reasoning alignment without additional context (as in REST-PG) also underperforms, since the model lacks sufficient personalized evidence to reason over. 

The results demonstrate that context expansion and reasoning alignment are complementary. With only context expansion (\method-r-ft), the model introduces noise and degrades performance. With only reasoning, even when aligned with finetuning (e.g., REST-PG), the model has nothing substantial to reason over. Only by combining both can \method achieve better personalization across text generation and rating prediction.

\subsubsection{Hyperparameter Analysis}
\label{exp:k}
The hyperparameter $K$ as defined in \cref{eq:tildeH} controls how many candidate items we add via the link predictor when augmenting a user’s personal context. In our main experiments, we fix K = 2 for efficiency, but the design of our method allows more effective and robust use with larger values of K. By comparison, PGraph expands context by retrieving K nearest reviews to the query embedding. However, because it lacks explicit reasoning alignment, its performance drops as K grows — the extra retrieved context introduces more noise than benefit (as seen in Table~\ref{tab:k_sensitivity}). 

Our theoretical analysis (detailed in Appendix~\ref{appendices:theory}) shows that this phenomenon follows a bias–variance trade-off. Adding more synthetic context reduces variance (helping especially sparse users with few real samples), but also introduces bias from preference mismatch. Without alignment, the bias term grows with K, limiting the benefit. However, with reasoning alignment, the mismatch is effectively shrunk: the bias is reduced, so larger K values become safe and beneficial. This explains why in experiments our method continues to improve as K increases, while PGraph plateaus or even degrades. Full mathematical details of this trade-off are provided in the \cref{prop:tradeoff}.

\begin{table*}[t]
\renewcommand{\arraystretch}{0.75}
    \centering
    \scriptsize
    \caption{Sensitivity of \method\ and PGraph to the neighborhood size $k$ on the Amazon Reviews dataset with the Llama3 backbone. Best in \textbf{bold}, second-best \underline{underlined}.}
    \label{tab:k_sensitivity}
    \begin{tabular}{llcccccccc}
        \toprule
        \multirow{2}{*}{\textbf{Task}} & \multirow{2}{*}{\textbf{Metric}} 
            & \multicolumn{4}{c}{\method} & \multicolumn{4}{c}{PGraph} \\
        \cmidrule(lr){3-6} \cmidrule(lr){7-10}
        & & \textbf{$K{=}1$} & \textbf{$K{=}2$} & \textbf{$K{=}3$} & \textbf{$K{=}4$} 
          & \textbf{$K{=}1$} & \textbf{$K{=}2$} & \textbf{$K{=}3$} & \textbf{$K{=}4$} \\
        \midrule
        \multirow{3}{*}{\textbf{Long Text Gen.}} 
            & ROUGE-1 & \textbf{0.226} & 0.215 & 0.222 & \underline{0.225} & 0.160 & 0.169 & \underline{0.169} & \textbf{0.173} \\
            & ROUGE-L & 0.172 & 0.171 & \underline{0.173} & \textbf{0.177} & 0.121 & \textbf{0.125} & 0.124 & 0.124 \\
            & METEOR  & 0.184 & 0.178 & \underline{0.189} & \textbf{0.196} & 0.125 & 0.138 & 0.149 & \textbf{0.150} \\
        \midrule
        \multirow{3}{*}{\textbf{Short Text Gen.}} 
            & ROUGE-1 & 0.168 & 0.155 & \underline{0.168} & \textbf{0.195} & \textbf{0.128} & 0.123 & 0.112 & 0.125 \\
            & ROUGE-L &\underline{0.173} & 0.153 & 0.165 & \textbf{0.192} & \textbf{0.121} & 0.118 & 0.107 & 0.119 \\
            & METEOR  & 0.160 & 0.142 & \underline{0.169} & \textbf{0.185} & 0.123 & \textbf{0.128} & 0.116 & 0.117 \\
        \bottomrule
    \end{tabular}
\end{table*}

\begin{figure*}[t]
    \centering
    \includegraphics[width=0.9\linewidth]{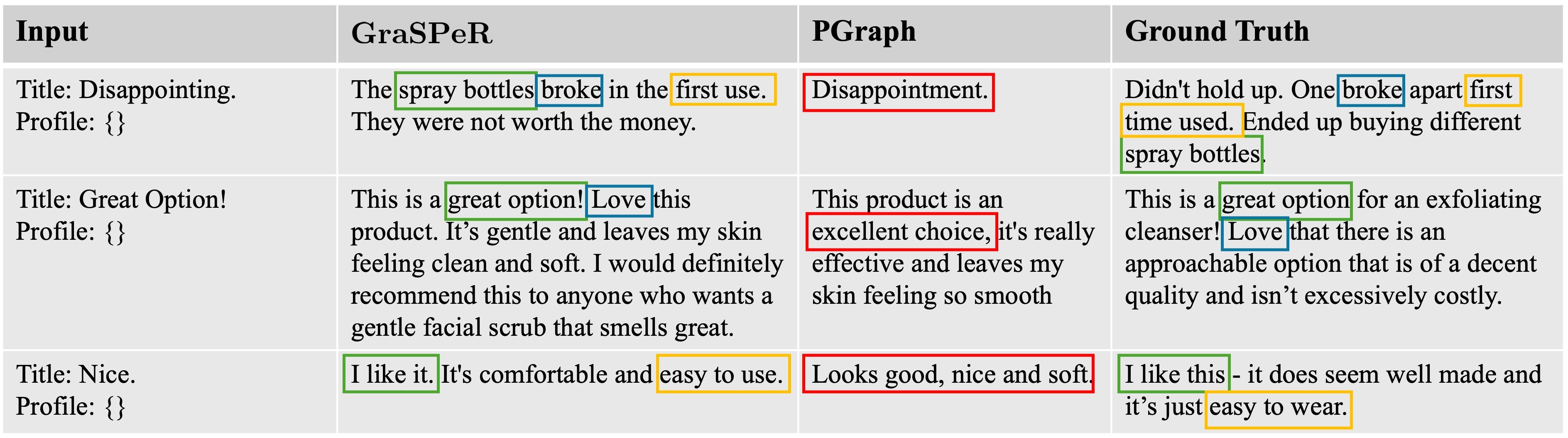}
    \caption{Case study with three examples. The matching green, blue, and yellow boxes show matching semantics or expression. The red box shows misalignment against the ground truth.}
    \label{fig:case_study}
    \vspace{-0.1in}
\end{figure*}

\subsection{Case Study}
To better illustrate how \method generates more faithful and personalized outputs, we present a case study in Figure~\ref{fig:case_study}. We compare outputs from \method and PGraph against the ground truth under different input titles. As highlighted by the colored boxes, \method consistently captures key semantics and stylistic expressions that align with the reference. For instance, in the first example, \method generates “spray bottles broke in the first use,” which mirrors both the semantics ("first use") and specific product mentioned (“spray bottles”) in the ground truth. In contrast, PGraph only outputs a vague summary (“Disappointment”) without grounding in the product context.

In the second example, \method reproduces stylistic markers such as “great option” and “Love”. Although PGraph's output "excellent choice" also captures the semantic meaning, it diverges from the ground truth in word choice. Finally, in the third example, both \method and the ground truth emphasize short, colloquial phrasing (``I like it / I like this''). \method also correctly matches the emphasis on usability (“easy to use / easy to wear”), while PGraph generates more general, less faithful wording (“Looks good, nice and soft”), deviating from the intended expression. These qualitative examples support our quantitative findings: context expansion and reasoning alignment together enable \method to preserve fine-grained semantics and stylistic fidelity, while methods that rely only on raw augmentation often produce generic or misaligned outputs. 
\section{Conclusion}
In this work, we proposed \method, a reasoning-based framework for personalized text generation under sparse user contexts. By combining graph-based context expansion with explicit reasoning alignment, our method effectively enriches limited personal histories while ensuring generated outputs remain faithful to user style and preferences. Extensive experiments across datasets in different domains demonstrate that \method significantly outperforms strong baselines.

Our findings highlight that context expansion and reasoning alignment are complementary: expansion alone risks introducing noise, while reasoning without sufficient context lacks grounding. Together, they enable models to better capture fine-grained semantics and stylistic fidelity, even for long-tail users with minimal histories. We believe this reasoning-enhanced paradigm opens promising directions for future research in large language model personalization, especially in real-world scenarios where data sparsity is the norm.

\bibliography{main}

@article{hou2024bridging,
  title={Bridging language and items for retrieval and recommendation},
  author={Hou, Yupeng and Li, Jiacheng and He, Zhankui and Yan, An and Chen, Xiusi and McAuley, Julian},
  journal={arXiv preprint arXiv:2403.03952},
  year={2024}
}

@inproceedings{ni-mcauley-2018-personalized,
    title = "Personalized Review Generation By Expanding Phrases and Attending on Aspect-Aware Representations",
    author = "Ni, Jianmo  and
      McAuley, Julian",
    editor = "Gurevych, Iryna  and
      Miyao, Yusuke",
    booktitle = "Proceedings of the 56th Annual Meeting of the Association for Computational Linguistics (Volume 2: Short Papers)",
    month = jul,
    year = "2018",
    address = "Melbourne, Australia",
    publisher = "Association for Computational Linguistics",
    url = "https://aclanthology.org/P18-2112/",
    doi = "10.18653/v1/P18-2112",
    pages = "706--711"
}

@misc{zhang2025personalizationlargelanguagemodels,
      title={Personalization of Large Language Models: A Survey}, 
      author={Zhehao Zhang and Ryan A. Rossi and Branislav Kveton and Yijia Shao and Diyi Yang and Hamed Zamani and Franck Dernoncourt and Joe Barrow and Tong Yu and Sungchul Kim and Ruiyi Zhang and Jiuxiang Gu and Tyler Derr and Hongjie Chen and Junda Wu and Xiang Chen and Zichao Wang and Subrata Mitra and Nedim Lipka and Nesreen Ahmed and Yu Wang},
      year={2025},
      eprint={2411.00027},
      archivePrefix={arXiv},
      primaryClass={cs.CL},
      url={https://arxiv.org/abs/2411.00027}, 
}

@misc{salemi2024lamplargelanguagemodels,
      title={LaMP: When Large Language Models Meet Personalization}, 
      author={Alireza Salemi and Sheshera Mysore and Michael Bendersky and Hamed Zamani},
      year={2024},
      eprint={2304.11406},
      archivePrefix={arXiv},
      primaryClass={cs.CL},
      url={https://arxiv.org/abs/2304.11406}, 
}

@inproceedings{10.1145/3459637.3482313,
author = {Zhao, Tong and Ni, Bo and Yu, Wenhao and Guo, Zhichun and Shah, Neil and Jiang, Meng},
title = {Action Sequence Augmentation for Early Graph-based Anomaly Detection},
year = {2021},
isbn = {9781450384469},
publisher = {Association for Computing Machinery},
address = {New York, NY, USA},
url = {https://doi.org/10.1145/3459637.3482313},
doi = {10.1145/3459637.3482313},
booktitle = {Proceedings of the 30th ACM International Conference on Information \& Knowledge Management},
pages = {2668–2678},
numpages = {11},
location = {Virtual Event, Queensland, Australia},
series = {CIKM '21}
}

@article{Zhao_Liu_Neves_Woodford_Jiang_Shah_2021, title={Data Augmentation for Graph Neural Networks}, volume={35}, url={https://ojs.aaai.org/index.php/AAAI/article/view/17315}, DOI={10.1609/aaai.v35i12.17315}, abstractNote={Data augmentation has been widely used to improve generalizability of machine learning models. However, comparatively little work studies data augmentation for graphs. This is largely due to the complex, non-Euclidean structure of graphs, which limits possible manipulation operations. Augmentation operations commonly used in vision and language have no analogs for graphs. Our work studies graph data augmentation for graph neural networks (GNNs) in the context of improving semi-supervised node-classification. We discuss practical and theoretical motivations, considerations and strategies for graph data augmentation. Our work shows that neural edge predictors can effectively encode class-homophilic structure to promote intra-class edges and demote inter-class edges in given graph structure, and our main contribution introduces the GAug graph data augmentation framework, which leverages these insights to improve performance in GNN-based node classification via edge prediction. Extensive experiments on multiple benchmarks show that augmentation via GAug improves performance across GNN architectures and datasets.}, number={12}, journal={Proceedings of the AAAI Conference on Artificial Intelligence}, author={Zhao, Tong and Liu, Yozen and Neves, Leonardo and Woodford, Oliver and Jiang, Meng and Shah, Neil}, year={2021}, month={May}, pages={11015-11023} }

@inproceedings{wegmann-etal-2022-author,
    title = "Same Author or Just Same Topic? Towards Content-Independent Style Representations",
    author = "Wegmann, Anna  and
      Schraagen, Marijn  and
      Nguyen, Dong",
    editor = "Gella, Spandana  and
      He, He  and
      Majumder, Bodhisattwa Prasad  and
      Can, Burcu  and
      Giunchiglia, Eleonora  and
      Cahyawijaya, Samuel  and
      Min, Sewon  and
      Mozes, Maximilian  and
      Li, Xiang Lorraine  and
      Augenstein, Isabelle  and
      Rogers, Anna  and
      Cho, Kyunghyun  and
      Grefenstette, Edward  and
      Rimell, Laura  and
      Dyer, Chris",
    booktitle = "Proceedings of the 7th Workshop on Representation Learning for NLP",
    month = may,
    year = "2022",
    address = "Dublin, Ireland",
    publisher = "Association for Computational Linguistics",
    url = "https://aclanthology.org/2022.repl4nlp-1.26/",
    doi = "10.18653/v1/2022.repl4nlp-1.26",
    pages = "249--268"
}

@misc{zhang2018linkpredictionbasedgraph,
      title={Link Prediction Based on Graph Neural Networks}, 
      author={Muhan Zhang and Yixin Chen},
      year={2018},
      eprint={1802.09691},
      archivePrefix={arXiv},
      primaryClass={cs.LG},
      url={https://arxiv.org/abs/1802.09691}, 
}

@misc{hamilton2018inductiverepresentationlearninglarge,
      title={Inductive Representation Learning on Large Graphs}, 
      author={William L. Hamilton and Rex Ying and Jure Leskovec},
      year={2018},
      eprint={1706.02216},
      archivePrefix={arXiv},
      primaryClass={cs.SI},
      url={https://arxiv.org/abs/1706.02216}, 
}

@misc{wang2025bridgingreviewsparsityrecommendation,
      title={Towards Bridging Review Sparsity in Recommendation with Textual Edge Graph Representation}, 
      author={Leyao Wang and Xutao Mao and Xuhui Zhan and Yuying Zhao and Bo Ni and Ryan A. Rossi and Nesreen K. Ahmed and Tyler Derr},
      year={2025},
      eprint={2508.01128},
      archivePrefix={arXiv},
      primaryClass={cs.IR},
      url={https://arxiv.org/abs/2508.01128}, 
}

@inproceedings{kanouchi-etal-2020-may,
    title = "You May Like This Hotel Because ...: Identifying Evidence for Explainable Recommendations",
    author = "Kanouchi, Shin  and
      Neishi, Masato  and
      Hayashibe, Yuta  and
      Ouchi, Hiroki  and
      Okazaki, Naoaki",
    editor = "Wong, Kam-Fai  and
      Knight, Kevin  and
      Wu, Hua",
    booktitle = "Proceedings of the 1st Conference of the Asia-Pacific Chapter of the Association for Computational Linguistics and the 10th International Joint Conference on Natural Language Processing",
    month = dec,
    year = "2020",
    address = "Suzhou, China",
    publisher = "Association for Computational Linguistics",
    url = "https://aclanthology.org/2020.aacl-main.89",
    doi = "10.18653/v1/2020.aacl-main.89",
    pages = "890--899",
}

@inproceedings{lin-2004-rouge,
    title = "{ROUGE}: A Package for Automatic Evaluation of Summaries",
    author = "Lin, Chin-Yew",
    booktitle = "Text Summarization Branches Out",
    month = jul,
    year = "2004",
    address = "Barcelona, Spain",
    publisher = "Association for Computational Linguistics",
    url = "https://aclanthology.org/W04-1013/",
    pages = "74--81"
}

@inproceedings{robertson1995okapi,
  title={Okapi at TREC-3},
  author={Robertson, Stephen E. and Walker, Steve and Jones, Susan and Hancock-Beaulieu, Micheline M. and Gatford, Mike},
  booktitle={Proceedings of the Third Text REtrieval Conference (TREC-3)},
  pages={109--126},
  year={1995}
}

@inproceedings{liu-etal-2023-g,
    title = "{G}-Eval: {NLG} Evaluation using Gpt-4 with Better Human Alignment",
    author = "Liu, Yang  and
      Iter, Dan  and
      Xu, Yichong  and
      Wang, Shuohang  and
      Xu, Ruochen  and
      Zhu, Chenguang",
    editor = "Bouamor, Houda  and
      Pino, Juan  and
      Bali, Kalika",
    booktitle = "Proceedings of the 2023 Conference on Empirical Methods in Natural Language Processing",
    month = dec,
    year = "2023",
    address = "Singapore",
    publisher = "Association for Computational Linguistics",
    url = "https://aclanthology.org/2023.emnlp-main.153/",
    doi = "10.18653/v1/2023.emnlp-main.153",
    pages = "2511--2522"
}

@misc{reimers2019sentencebertsentenceembeddingsusing,
      title={Sentence-BERT: Sentence Embeddings using Siamese BERT-Networks}, 
      author={Nils Reimers and Iryna Gurevych},
      year={2019},
      eprint={1908.10084},
      archivePrefix={arXiv},
      primaryClass={cs.CL},
      url={https://arxiv.org/abs/1908.10084}, 
}

@inproceedings{alhafni-etal-2024-personalized,
    title = "Personalized Text Generation with Fine-Grained Linguistic Control",
    author = "Alhafni, Bashar  and
      Kulkarni, Vivek  and
      Kumar, Dhruv  and
      Raheja, Vipul",
    editor = "Deshpande, Ameet  and
      Hwang, EunJeong  and
      Murahari, Vishvak  and
      Park, Joon Sung  and
      Yang, Diyi  and
      Sabharwal, Ashish  and
      Narasimhan, Karthik  and
      Kalyan, Ashwin",
    booktitle = "Proceedings of the 1st Workshop on Personalization of Generative AI Systems (PERSONALIZE 2024)",
    month = mar,
    year = "2024",
    address = "St. Julians, Malta",
    publisher = "Association for Computational Linguistics",
    url = "https://aclanthology.org/2024.personalize-1.8",
    pages = "88--101",
}

@ARTICLE{6648327,
  author={Qian, Xueming and Feng, He and Zhao, Guoshuai and Mei, Tao},
  journal={IEEE Transactions on Knowledge and Data Engineering}, 
  title={Personalized Recommendation Combining User Interest and Social Circle}, 
  year={2014},
  volume={26},
  number={7},
  pages={1763-1777},
  doi={10.1109/TKDE.2013.168}}

@article{SHUMANOV2021106627,
title = {Making conversations with chatbots more personalized},
journal = {Computers in Human Behavior},
volume = {117},
pages = {106627},
year = {2021},
issn = {0747-5632},
doi = {https://doi.org/10.1016/j.chb.2020.106627},
url = {https://www.sciencedirect.com/science/article/pii/S0747563220303745},
author = {Michael Shumanov and Lester Johnson}
}

@article{Yoganarasimhan2019,
  author    = {Hema Yoganarasimhan},
  title     = {Search Personalization Using Machine Learning},
  journal   = {Management Science},
  volume    = {66},
  number    = {3},
  pages     = {1045-1070},
  year      = {2019},
  doi       = {10.1287/mnsc.2018.3255},
  publisher = {INFORMS}
}

@misc{kumar2024longlampbenchmarkpersonalizedlongform,
      title={LongLaMP: A Benchmark for Personalized Long-form Text Generation}, 
      author={Ishita Kumar and Snigdha Viswanathan and Sushrita Yerra and Alireza Salemi and Ryan A. Rossi and Franck Dernoncourt and Hanieh Deilamsalehy and Xiang Chen and Ruiyi Zhang and Shubham Agarwal and Nedim Lipka and Chien Van Nguyen and Thien Huu Nguyen and Hamed Zamani},
      year={2024},
      eprint={2407.11016},
      archivePrefix={arXiv},
      primaryClass={cs.CL},
      url={https://arxiv.org/abs/2407.11016}, 
}

@misc{wei2023chainofthoughtpromptingelicitsreasoning,
      title={Chain-of-Thought Prompting Elicits Reasoning in Large Language Models}, 
      author={Jason Wei and Xuezhi Wang and Dale Schuurmans and Maarten Bosma and Brian Ichter and Fei Xia and Ed Chi and Quoc Le and Denny Zhou},
      year={2023},
      eprint={2201.11903},
      archivePrefix={arXiv},
      primaryClass={cs.CL},
      url={https://arxiv.org/abs/2201.11903}, 
}

@misc{wang2023selfconsistencyimproveschainthought,
      title={Self-Consistency Improves Chain of Thought Reasoning in Language Models}, 
      author={Xuezhi Wang and Jason Wei and Dale Schuurmans and Quoc Le and Ed Chi and Sharan Narang and Aakanksha Chowdhery and Denny Zhou},
      year={2023},
      eprint={2203.11171},
      archivePrefix={arXiv},
      primaryClass={cs.CL},
      url={https://arxiv.org/abs/2203.11171}, 
}

@misc{yao2023treethoughtsdeliberateproblem,
      title={Tree of Thoughts: Deliberate Problem Solving with Large Language Models}, 
      author={Shunyu Yao and Dian Yu and Jeffrey Zhao and Izhak Shafran and Thomas L. Griffiths and Yuan Cao and Karthik Narasimhan},
      year={2023},
      eprint={2305.10601},
      archivePrefix={arXiv},
      primaryClass={cs.CL},
      url={https://arxiv.org/abs/2305.10601}, 
}

@inproceedings{tsai-etal-2024-leveraging,
    title = "Leveraging {LLM} Reasoning Enhances Personalized Recommender Systems",
    author = "Tsai, Alicia  and
      Kraft, Adam  and
      Jin, Long  and
      Cai, Chenwei  and
      Hosseini, Anahita  and
      Xu, Taibai  and
      Zhang, Zemin  and
      Hong, Lichan  and
      Chi, Ed H.  and
      Yi, Xinyang",
    editor = "Ku, Lun-Wei  and
      Martins, Andre  and
      Srikumar, Vivek",
    booktitle = "Findings of the Association for Computational Linguistics: ACL 2024",
    month = aug,
    year = "2024",
    address = "Bangkok, Thailand",
    publisher = "Association for Computational Linguistics",
    url = "https://aclanthology.org/2024.findings-acl.780/",
    doi = "10.18653/v1/2024.findings-acl.780",
    pages = "13176--13188"
}

@misc{luo2025reasoningmeetspersonalizationunleashing,
      title={Reasoning Meets Personalization: Unleashing the Potential of Large Reasoning Model for Personalized Generation}, 
      author={Sichun Luo and Guanzhi Deng and Jian Xu and Xiaojie Zhang and Hanxu Hou and Linqi Song},
      year={2025},
      eprint={2505.17571},
      archivePrefix={arXiv},
      primaryClass={cs.CL},
      url={https://arxiv.org/abs/2505.17571}, 
}

@misc{lyu2024llmrecpersonalizedrecommendationprompting,
      title={LLM-Rec: Personalized Recommendation via Prompting Large Language Models}, 
      author={Hanjia Lyu and Song Jiang and Hanqing Zeng and Yinglong Xia and Qifan Wang and Si Zhang and Ren Chen and Christopher Leung and Jiajie Tang and Jiebo Luo},
      year={2024},
      eprint={2307.15780},
      archivePrefix={arXiv},
      primaryClass={cs.CL},
      url={https://arxiv.org/abs/2307.15780}, 
}

@misc{bismay2024reasoningrecbridgingpersonalizedrecommendations,
      title={ReasoningRec: Bridging Personalized Recommendations and Human-Interpretable Explanations through LLM Reasoning}, 
      author={Millennium Bismay and Xiangjue Dong and James Caverlee},
      year={2024},
      eprint={2410.23180},
      archivePrefix={arXiv},
      primaryClass={cs.IR},
      url={https://arxiv.org/abs/2410.23180}, 
}

@misc{yang2023palrpersonalizationawarellms,
      title={PALR: Personalization Aware LLMs for Recommendation}, 
      author={Fan Yang and Zheng Chen and Ziyan Jiang and Eunah Cho and Xiaojiang Huang and Yanbin Lu},
      year={2023},
      eprint={2305.07622},
      archivePrefix={arXiv},
      primaryClass={cs.IR},
      url={https://arxiv.org/abs/2305.07622}, 
}

@misc{kim2025llmsthinkflowreasoninglevel,
      title={LLMs Think, But Not In Your Flow: Reasoning-Level Personalization for Black-Box Large Language Models}, 
      author={Jieyong Kim and Tongyoung Kim and Soojin Yoon and Jaehyung Kim and Dongha Lee},
      year={2025},
      eprint={2505.21082},
      archivePrefix={arXiv},
      primaryClass={cs.CL},
      url={https://arxiv.org/abs/2505.21082}, 
}

@article{Besta_2024,
   title={Graph of Thoughts: Solving Elaborate Problems with Large Language Models},
   volume={38},
   ISSN={2159-5399},
   url={http://dx.doi.org/10.1609/aaai.v38i16.29720},
   DOI={10.1609/aaai.v38i16.29720},
   number={16},
   journal={Proceedings of the AAAI Conference on Artificial Intelligence},
   publisher={Association for the Advancement of Artificial Intelligence (AAAI)},
   author={Besta, Maciej and Blach, Nils and Kubicek, Ales and Gerstenberger, Robert and Podstawski, Michal and Gianinazzi, Lukas and Gajda, Joanna and Lehmann, Tomasz and Niewiadomski, Hubert and Nyczyk, Piotr and Hoefler, Torsten},
   year={2024},
   month=mar, pages={17682–17690} }

@misc{li_system_2025,
	title = {From System 1 to System 2: A Survey of Reasoning Large Language Models},
	url = {http://arxiv.org/abs/2502.17419},
	doi = {10.48550/arXiv.2502.17419},
	shorttitle = {From System 1 to System 2},
	number = {{arXiv}:2502.17419},
	publisher = {{arXiv}},
	author = {Li, Zhong-Zhi and Zhang, Duzhen and Zhang, Ming-Liang and Zhang, Jiaxin and Liu, Zengyan and Yao, Yuxuan and Xu, Haotian and Zheng, Junhao and Wang, Pei-Jie and Chen, Xiuyi and Zhang, Yingying and Yin, Fei and Dong, Jiahua and Li, Zhiwei and Bi, Bao-Long and Mei, Ling-Rui and Fang, Junfeng and Guo, Zhijiang and Song, Le and Liu, Cheng-Lin},
	urldate = {2025-05-19},
	date = {2025-04-25},
	eprinttype = {arxiv},
	eprint = {2502.17419 [cs]},
}

@misc{au_personalized_2025,
	title = {Personalized Graph-Based Retrieval for Large Language Models},
	url = {http://arxiv.org/abs/2501.02157},
	doi = {10.48550/arXiv.2501.02157},
	number = {{arXiv}:2501.02157},
	publisher = {{arXiv}},
	author = {Au, Steven and Dimacali, Cameron J. and Pedirappagari, Ojasmitha and Park, Namyong and Dernoncourt, Franck and Wang, Yu and Kanakaris, Nikos and Deilamsalehy, Hanieh and Rossi, Ryan A. and Ahmed, Nesreen K.},
	urldate = {2025-05-19},
	date = {2025-01-04},
year={2025},
	eprinttype = {arxiv},
	eprint = {2501.02157 [cs]},
}

@misc{salemi_reasoning-enhanced_2025,
      title={Reasoning-Enhanced Self-Training for Long-Form Personalized Text Generation}, 
      author={Alireza Salemi and Cheng Li and Mingyang Zhang and Qiaozhu Mei and Weize Kong and Tao Chen and Zhuowan Li and Michael Bendersky and Hamed Zamani},
      year={2025},
      eprint={2501.04167},
      archivePrefix={arXiv},
      primaryClass={cs.CL},
      url={https://arxiv.org/abs/2501.04167}, 
}

@misc{salemi2023lamp,
      title={La{MP}: When Large Language Models Meet Personalization}, 
      author={Alireza Salemi and Sheshera Mysore and Michael Bendersky and Hamed Zamani},
      year={2023},
      eprint={2304.11406},
      archivePrefix={arXiv},
      primaryClass={cs.CL}
}

@misc{salemi2024optimization,
      title={Optimization Methods for Personalizing Large Language Models through Retrieval Augmentation}, 
      author={Alireza Salemi and Surya Kallumadi and Hamed Zamani},
      year={2024},
      eprint={2404.05970},
      archivePrefix={arXiv},
      primaryClass={cs.CL}
}
\bibliographystyle{iclr2025_conference}

\newpage
\appendix
\section{Related Work}

\paragraph{LLM Personalization.} Personalization in LLMs has recently garnered significant attention~\cite{salemi2024lamplargelanguagemodels, zhang2025personalizationlargelanguagemodels} due to its potential to improve various downstream applications, including search, recommendation, and conversational agents~\cite{Yoganarasimhan2019, 6648327, SHUMANOV2021106627}. \citet{salemi2023lamp} introduced the LaMP benchmark, which comprises seven datasets designed to evaluate personalization in language models by incorporating personal context into downstream generation and classification tasks. Building on this, \citet{salemi2024optimization} explores optimization strategies for personalization by improving the retriever’s ability to select relevant personal context. Furthermore, \citet{kumar2024longlampbenchmarkpersonalizedlongform} extends personalization research to the domain of long-form text generation. In addition, to address the issue of long-tail data sparsity in the personal context histories, ~\citet{au_personalized_2025} proposes to augment the personal context with user-centric graphs, retrieving relevant histories from other users.  

\paragraph{LLM Reasoning and Planning.} Reasoning in LLM encourages the language model to think and plan before generation, leading to more coherent and accurate outputs~\cite{li_system_2025}. Chain-of-Thought (CoT) prompting~\cite{wei2023chainofthoughtpromptingelicitsreasoning} was first proposed to elicit reasoning capabilities of language models by prompting the model to generate a series of intermediate steps that lead to a final answer. Various methods have extended the CoT prompting to address its deficiencies. For example, Self-Consistency~\cite{wang2023selfconsistencyimproveschainthought} samples multiple reasoning paths and selects the most consistent answer, mitigating the impact of occasional reasoning errors. Additionally, Tree-of-Thoughts(ToT)~\cite{yao2023treethoughtsdeliberateproblem} allow LLMs to explore multiple reasoning paths in a tree-like structure, performing deliberate lookahead and backtracking to make more informed decisions. Graph-of-Thoughts(GoT)~\cite{Besta_2024} further generalizes the ToT by modeling the reasoning processes as arbitrary graphs. 

Recently, several works have explored reasoning in personalization. REST-PG~\cite{salemi_reasoning-enhanced_2025} employs self-training on LLM reasoning paths to improve the personalization. ~\citet{luo2025reasoningmeetspersonalizationunleashing} explored reinforced reasoning for personalization by incorporating and refining a hierarchical reasoning thought template to guide the reasoning process. Additionally, ~\citet{kim2025llmsthinkflowreasoninglevel} explored reasoning-level personalization that aligns model's reasoning process with a user's personalized logic. Several works have also explored reasoning to enhance personalized recommendations~\cite{lyu2024llmrecpersonalizedrecommendationprompting, bismay2024reasoningrecbridgingpersonalizedrecommendations, yang2023palrpersonalizationawarellms} beyond traditional item-based recommendations. 

\section{LLM-as-a-Judge}
\label{appendix:lj}
Traditionally, in the prior personalization benchmarks~\citep{au_personalized_2025, kumar2024longlampbenchmarkpersonalizedlongform, salemi2023lamp}, personalized text generation has been evaluated with lexical overlap metrics such as ROUGE~\citep{lin-2004-rouge}. However, it has been shown that such metrics may fail to capture the semantic nuances and stylistic alignment in personalization. Thus, we adopt the LLM-as-a-Judge prompt from prior works on personalized text generation~\cite{salemi_reasoning-enhanced_2025}, which is designed based on the evaluation paradigm introduced in \cite{liu-etal-2023-g}. Our prompt is introduced as follows.
\newpage 
\begin{mybox}{LLM-as-a-Judge}
Please compare the generated text to the reference text based on how well they match and/or are similar.
\\\\
Scoring Scale:\\
 1 – Strongly disagree \\
 2 – Disagree\\
 3 – Somewhat disagree\\
 4 – Neither agree nor disagree\\
 5 – Somewhat agree\\
 6 – Agree \\
 7 – Strongly agree\\
\\\\
Content to Evaluate: \\
Reference Text (Ground Truth): \{target\_text\}\\
Generated Text: \{generated\_text\}\\\\

Provide only the numeric score (1–7).
\end{mybox}
\vspace{1ex}

We use GPT-4 as the judge LLM, and report the normalized score (0.1-0.7) in our main experiment table. ~\cite{salemi_reasoning-enhanced_2025} designed additional experiments to validate the effectiveness of the LLM-as-a-Judge evaluation. First, they conduct a human evaluation comparing 100 model outputs and find that the LLM-as-a-Judge scores agree with human preference in 73\% of cases, with a Pearson correlation of 0.46. Second, they design a controlled perturbation study by randomly replacing a portion of the personalized contexts with unrelated ones. The LLM-as-a-Judge scores decrease linearly as the perturbation rate increases, showing that the evaluator is sensitive to mismatched personalization.

While no automatic metric can fully replicate human evaluation for personalization—since the “true” judge of style and preference is the original user—LLM-as-a-Judge provides a scalable and semantically meaningful proxy. In our setting, it enables consistent evaluation across sparse and noisy contexts, capturing personalization quality beyond what lexical metrics can measure.

\section{Experimental Setup}
\label{appendices:exp-set-up}
\subsection{Dataset Statistics}
\label{appendices:dataset}
\begin{table}[h!]
\centering
\small
\resizebox{0.6\columnwidth}{!}{
\begin{tabular}{lccc}
\toprule
\textbf{Dataset} & \textbf{Train Size} & \textbf{Validation Size} & \textbf{Test Size} \\
\midrule
User-Product Review & 20,000 & 2,500 & 2,500 \\
Stylized Feedback & 20,000 & 2,500 & 2,500 \\
Hotel Experiences & 9,000 & 2,500 & 2,500 \\
\bottomrule
\end{tabular}}
\caption{Dataset split sizes across training, validation, and test sets for the four domains.}
\label{tab:split-stats}
\end{table}

We provide the dataset statistics in this section. In \cref{tab:split-stats}, we give the train/validation/test split statistics for the datasets. It is worth noting that the Hotel Experience dataset is a smaller dataset with a smaller training set, leading to the more inconsistent performance that we presented in the Experiment section. In \cref{tab:task-stats}, we introduce the task statistics for Long Text Generation, Short Text Generation, and Ordinal Classification. The datasets are constructed to reflect the real-world distribution~\citep{au_personalized_2025}, which results in the sparse profiles as shown in the Average Profile Size. \method achieves more consistent and significant performance gain in scenarios where the output length is shorter, such as Short Text Generation and the User-Product Review (Amazon dataset), as longer text implicitly gives more context for text generation.
\begin{table*}[h!]
    \centering
    \begin{adjustbox}{width=1.0\textwidth}
        \begin{tabular}{lccccc} 
        \toprule
            \textbf{Task} & \textbf{Type} & \textbf{Avg. Input Length} & \textbf{Avg. Output Length} & \textbf{Avg. Profile Size} & \textbf{\# Classes} \\
        \midrule
            User-Product Review Generation & Long Text Generation & $3.754\pm2.71$ & $47.90\pm19.28$ & $1.05\pm0.31$ & - \\
            Hotel Experiences Generation & Long Text Generation & $4.29\pm2.57$ & $76.26\pm22.39$ & $1.14\pm0.61$ & - \\
            Stylized Feedback Generation & Long Text Generation & $3.35\pm2.02$ & $51.80\pm20.07$ & $1.09\pm0.47$ & - \\
        \midrule
            User-Product Review Title Generation & Short Text Generation & $30.34\pm37.95$ & $7.02\pm1.14$ & $1.05\pm0.31$ & - \\
            Hotel Experiences Summary Generation & Short Text Generation & $90.40\pm99.17$ & $7.64\pm0.92$ & $1.14\pm0.61$ & - \\
            Stylized Feedback Title Generation & Short Text Generation & $37.42\pm38.17$ & $7.16\pm1.11$ & $1.09\pm0.47$ & - \\
        \midrule
            User-Product Review Ratings & Ordinal Classification & $34.10\pm38.66$  & - & $1.05\pm0.31$ & 5 \\
            Hotel Experiences Ratings & Ordinal Classification & $94.69\pm99.62$  & - & $1.14\pm0.61$ & 5 \\
                Stylized Feedback Ratings & Ordinal Classification & $40.77\pm38.69$ & - & $1.09\pm0.47$ & 5 \\
        \bottomrule
        \end{tabular}
    \end{adjustbox}
    \caption{
    Data statistics for the PGraphRAG Benchmark across the four datasets. For each task, we report the average input and output lengths (in words), measured on the test set using BM25-based retrieval with GPT. The average profile size indicates the number of reviews per user used for personalization.
    }

    \label{tab:task-stats}
\end{table*}

\subsection{Datasets}
We evaluate our approach on three benchmark datasets introduced in prior research~\cite{au_personalized_2025}. These datasets cover diverse domains and graph structures, enabling us to assess the effectiveness of our method. 

\paragraph{Amazon Review.}
The Amazon Review dataset is constructed from the Amazon Review 2023 corpus~\citep{hou2024bridging}. We build a user-item interaction graph where nodes represent users and products, and edges indicate review interactions between them. 

\paragraph{Hotel Experience.}
The Hotel Experience dataset is collected from the Datafiniti Hotel Reviews dataset~\citep{au_personalized_2025}. It contains user-hotel interaction data, where edges denote users' stays at hotels and are annotated with textual reviews. 

\paragraph{Stylized Feedback Review.}
The Stylized Feedback Review dataset is derived from the Datafiniti Grammar and Online Product dataset~\citep{au_personalized_2025}. It focuses on generating stylistic and domain-specific feedback from user-product interactions. This dataset emphasizes linguistic diversity and style adaptation.

\subsection{Tasks}
Here we present an extended discussion on the tasks that we used to evaluate \method: Long Text Generation, Short Text Generation, and Ordinal Classification. 

\paragraph{Long Text Generation.}
The long text generation task focuses on producing detailed user reviews given a review title and the user's profile. The objective is to generate coherent and contextually relevant review text that aligns with the user’s preferences. This task evaluates the model’s capability for generating high-quality, personalized text.

\paragraph{Short Text Generation.}
The short text generation task involves generating concise product titles or summaries given a user review. The challenge lies in distilling a longer text into a shorter title. This task assesses the model’s ability to distill information from highly personalized user context.

\paragraph{Ordinal Classification.}
The ordinal classification task aims to predict the rating score a user would assign to a product based on the title and review text. This task is particularly challenging because of varying rating behaviors; for example, some users might write critical reviews yet still assign high scores. This task is designed to evaluate the model's ability to capture subtle patterns in user preferences and rating tendencies.

\vspace{4ex}
\subsection{Metrics}
\label{app:metrics}
\paragraph{Text Generation.} 
For both long and short text generation tasks, we adopt widely used lexical overlap metrics, including ROUGE-1 and ROUGE-L, following prior work~\cite{au_personalized_2025}. These metrics capture n-gram and subsequence overlaps between the generated output and ground-truth references. To complement these surface-level measures, we further incorporate LLM-as-a-Judge evaluation, where a strong language model provides comparative assessments of personalization and accuracy. We design the prompt based on prior studies which has been validated with human evaluators on the task of personalization~\citep{salemi_reasoning-enhanced_2025}. The prompt for LLM-as-a-Judge evaluation is provided in Appendix \ref{appendix:lj}.

\paragraph{Ordinal Classification.} 
For the ordinal classification task, we evaluate rating prediction using Root Mean Squared Error (RMSE) and Mean Absolute Error (MAE). RMSE penalizes large deviations more heavily, highlighting extreme mispredictions, while MAE measures the average magnitude of prediction errors.

\section{Pseudo Code}
\begin{algorithm}[h!]
\caption{GRASPER --- Training}
\label{alg:grasper-train}
\begin{algorithmic}[1]
\Require Bipartite graph $G=(U \cup I, E)$; user histories $\{H_u\}$; item reviews $\{R_i\}$; encoder $\text{Enc}(\cdot)$; base LLM $\mathcal{M}$; hyperparameters: $K$ (items to augment), $k_{\mathrm{sim}}$ (similar users), $k_{\mathrm{peer}}$ (peer texts)
\Ensure Trained link predictor (GraphSAGE + MLP), fine-tuned LLM $M'$
\Statex

\State \textbf{// Step 1: Personal Context Expansion}
\For{each node $v \in U \cup I$}
  \If{$v$ is user $u$} \State $h_v^{(0)} \gets \text{Enc}(\text{concat}(H_u))$ \EndIf
  \If{$v$ is item $i$} \State $h_v^{(0)} \gets \text{Enc}(R_i)$ \EndIf
\EndFor
\For{$\ell = 1$ to $L$} \Comment{GraphSAGE layers}
  \State $m_v^{(\ell)} \gets \mathbf{AGG}_{\ell}\big(\{\,h_{u}^{(\ell-1)} : u \in \mathcal{N}(v)\,\}\big)$
  \State $h_v^{(\ell)} \gets \mathrm{ReLU}\!\big(W_{\ell}[\,h_v^{(\ell-1)} \Vert m_v^{(\ell)}\,]\big)$
\EndFor
\State $z_v \gets h_v^{(L)}$ for all $v$
\State Score edges with $s(u,i)=\mathrm{MLP}([z_u \Vert z_i])$, $\hat{y}(u,i)=\sigma(s(u,i))$
\State Optimize BCE with negative sampling to train (GraphSAGE+MLP) $\rightarrow$ Link Predictor

\Statex
\State \textbf{// Step 2: Synthetic Review Generation with Reasoning Alignment}
\For{each training user $u$}
  \State $\mathcal{S}_u \gets \mathrm{TopK}_{k_{\mathrm{sim}}}\big(\cos(z_u, z_{\cdot})\big)$
  \State $H_{\mathcal{S}_u} \gets \{\text{reviews from users in } \mathcal{S}_u\}$
  \For{each observed pair $(u,j)$}
    \State $P_{u,j} \gets \mathrm{BM25\_TopK}_{k_{\mathrm{peer}}}(\text{reviews of } j)$
    \State $x \gets \{\,H_u \setminus \{t_{u,j}\},\; H_{\mathcal{S}_u},\; P_{u,j}\,\}$
    \State Sample candidate reasoning paths $\{Z^{(1)},\dots,Z^{(K)}\} \sim M$ with prompt $\phi(x)$
    \State $Z^\star \gets \arg\max_{Z} \;\Omega\big(M(\xi(x,Z)),\, t_{u,j}\big)$ \Comment{$\Omega$: dev metric (e.g., ROUGE/METEOR)}
    \State Update $M$ by minimizing $\text{CrossEntropy}\!\big(M(\rho(x)),\; [Z^\star \;\Vert\; t_{u,j}]\big)$
  \EndFor
\EndFor
\State $M' \gets M$
\State \Return (Link Predictor), $M'$
\end{algorithmic}
\end{algorithm}

\begin{algorithm}[h!]
\caption{GRASPER --- Inference}
\label{alg:grasper-infer}
\begin{algorithmic}[1]
\Require Trained (GraphSAGE+MLP), $M'$; graph $G$; $\{H_u\}$, $\{R_i\}$; $K$, $k_{\mathrm{sim}}$, $k_{\mathrm{peer}}$; target $(u,i^\star)$
\Ensure Personalized review $\hat{t}_{u,i^\star}$
\Statex

\State \textbf{// Step 1: Personal context expansion}
\State Initialize $h_v^{(0)}$ with $\text{Enc}(\cdot)$; run GraphSAGE to obtain $z_v$ for all $v$
\State Rank items $i \in I\setminus\{i:(u,i)\in E\}$ by $s(u,i)$; let $\mathcal{I}_u^K \gets \mathrm{TopK}_K$
\State $\mathcal{S}_u \gets \mathrm{TopK}_{k_{\mathrm{sim}}}\big(\cos(z_u, z_{\cdot})\big)$
\For{each $i \in \mathcal{I}_u^K$}
  \State $P_{u,i} \gets \mathrm{BM25\_TopK}_{k_{\mathrm{peer}}}(\text{reviews of } i)$
  \State $x_i \gets \{\,H_u,\; H_{\mathcal{S}_u},\; P_{u,i}\,\}$
  \State $[z_i' \;\Vert\; \tilde{t}_{u,i}] \gets M'(x_i)$ \Comment{reasoning + synthetic review}
\EndFor
\State $\tilde{H}_u \gets H_u \cup \{\tilde{t}_{u,i}: i \in \mathcal{I}_u^K\}$

\Statex
\State \textbf{// Step 2: Final personalized generation for target item}
\State $P_{u,i^\star} \gets \mathrm{BM25\_TopK}_{k_{\mathrm{peer}}}(\text{reviews of } i^\star)$
\State $x^\star \gets \{\,\tilde{H}_u,\; H_{\mathcal{S}_u},\; P_{u,i^\star}\,\}$
\State $[z^\star \;\Vert\; \hat{t}_{u,i^\star}] \gets M'(x^\star)$
\State \Return $\hat{t}_{u,i^\star}$
\end{algorithmic}
\end{algorithm}

\begin{table*}[t]
\centering
\scriptsize
\caption{Backbone ablation on Amazon Reviews dataset for \method\ using open-source (Llama 3, Gemma 2) and proprietary (GPT-4o mini, GPT-4.1) backbones. 
Metrics for text/title generation are higher-is-better; for rating prediction, lower-is-better.}
\label{tab:backbone_ablation}
\begin{adjustbox}{max width=\textwidth}
\begin{tabular}{lllcccccccccc}
    \toprule
    \textbf{Category} & \textbf{Backbone} & \textbf{Method} 
    & \multicolumn{4}{c}{\textbf{Text Generation}} 
    & \multicolumn{4}{c}{\textbf{Title Generation}} 
    & \multicolumn{2}{c}{\textbf{Rating Prediction}} \\
    \cmidrule(lr){4-7} \cmidrule(lr){8-11} \cmidrule(lr){12-13}
    & & & R-1 & R-L & MET & LJ & R-1 & R-L & MET & LJ & RMSE & MAE \\
    \midrule
    \multirow{4}{*}{Open-Source} 
        & \multirow{2}{*}{Llama 3} 
            & \method & \textbf{0.215} & \textbf{0.171} & \textbf{0.178} & \textbf{0.337} & 0.155 & \textbf{0.153} & 0.142 & \textbf{0.304} & \textbf{0.32} & \textbf{0.31} \\
        & & PGraph  & 0.178 & 0.129 & 0.151 & 0.297 & \textbf{0.178} & 0.129 & \textbf{0.151} & 0.241 & 0.76 & 0.34 \\
        \cmidrule(lr){2-13}
        & \multirow{2}{*}{Gemma 2} 
            & \method & \textbf{0.160} & \textbf{0.119} & \textbf{0.121} & \textbf{0.326} & \textbf{0.127} & \textbf{0.123} & 0.122 & \textbf{0.329} & \textbf{0.52} & \textbf{0.53} \\
        & & PGraph  & 0.155 & 0.119 &  0.117 & 0.316 & 0.098 & 0.093 & \textbf{0.124} & 0.297 & 0.88 & 0.42 \\
    \midrule
    \multirow{4}{*}{Proprietary} 
        & \multirow{2}{*}{GPT-4o mini} 
            & \method & \textbf{0.219} & \textbf{0.170} & 0.182 & \textbf{0.421} & \textbf{0.178} & \textbf{0.174} & \textbf{0.162} & \textbf{0.406} & \textbf{0.33} & \textbf{0.34} \\
        & & PGraph  & 0.189 & 0.130 & \textbf{0.196} & 0.389 & 0.115 & 0.112 & 0.099 & 0.353 & 0.38 & 0.73 \\
        \cmidrule(lr){2-13}
        & \multirow{2}{*}{GPT-4.1} 
            & \method & \textbf{0.221} & \textbf{0.176} & 0.181 & \textbf{0.433} & \textbf{0.151} & \textbf{0.150} & \textbf{0.166} & \textbf{0.401} & \textbf{0.31} & \textbf{0.32} \\
        & & PGraph  & 0.185 & 0.128 & \textbf{0.191} & 0.403 & 0.107 & 0.103 & 0.122 & 0.346 & 0.38 & 0.70 \\
    \bottomrule
\end{tabular}
\end{adjustbox}
\end{table*}

\section{Additonal Experiment Results}
\label{appendices:additional_results}
\subsection{Language Model Variants}
\label{sec:lm_variants}

In \cref{tab:backbone_ablation}, we compare \method across different backbone models, covering both open-source (Llama 3, Gemma 2) and proprietary (GPT-4o mini, GPT-4.1) variants. This setup allows us to test whether the improvements of \method depend on a particular language model family or extend across architectures with varying sizes and training pipelines.

For open-source models, \method consistently improves over the baseline PGraph across all metrics. With Llama 3, \method achieves a clear gain in text generation and rating prediction. Gemma-2, though smaller in scale, still benefits from our framework, showing improved semantic quality on LLM-as-a-Judge metric. These results suggest that \method effectively enhances smaller open-source models, making them more competitive for personalization tasks.

When applied to proprietary models, the improvements remain consistent. On GPT-4o mini, \method outperforms the baseline in text generation and especially in LLM-as-a-Judge, demonstrating better alignment with human preferences. GPT-4.1 mini, the more advanced backbone, also benefits: \method achieves the highest score across metrics, indicating strong personalization quality even when starting from a more powerful model.

Overall, the results confirm that \method is robust to the choice of language model backbone. Gains are observed consistently across both open-source and proprietary families. Importantly, improvements in LLM-as-a-Judge are more significant, underscoring that our framework aligns better with human preference. This robustness highlights \method’s practicality, as it can be flexibly deployed in various settings with different backbone models.

\subsection{\rebuttal{Link Prediction Noise and Robustness}}
Although \method achieves strong improvements on personalized text generation—particularly in sparse-user settings—the link prediction module inevitably introduces a degree of noise due to imperfect edge predictions. It is therefore important to assess both the quality of the predicted user–item links and the robustness of the downstream reasoning-based personalization to such noise.

\begin{table}[t]
    \centering
    \small
    \caption{Link Prediction Performance across Different Datasets}
    \label{tab:link_pred}
    \begin{tabular}{l c c c c}
        \toprule 
        \textbf{Dataset} & \textbf{MRR} & \textbf{Hits@1} & \textbf{Hits@5} & \textbf{Hits@10} \\
        \midrule 
        Amazon & 0.531 & 0.415 & 0.659 & 0.760 \\
        Hotel & 0.324 & 0.210 & 0.446 & 0.546 \\
        Feedbacks & 0.275 & 0.178 & 0.394 & 0.477 \\
        \bottomrule 
    \end{tabular}
\end{table}

Table~\ref{tab:link_pred} reports the standalone performance of the link prediction module across all datasets. The module demonstrates consistently strong ranking metrics, indicating its ability to recover meaningful user–item affinities even under sparse supervision. Nonetheless, some level of incorrect or low-confidence predictions is unavoidable. To study whether such noise impacts the final generation quality, we further conduct an analysis on the Amazon test set by partitioning examples into two groups: the top 50\% and bottom 50\% based on their link-prediction confidence scores.

\begin{table}[htbp]
    \centering
    \small
    \caption{Comparison of Text Generation Scores by Link Prediction Performance}
    \label{tab:link_pred_confidence}
    \begin{tabular}{l c c}
        \toprule 
        \textbf{Group} & \textbf{ROUGE-L Mean} & \textbf{METEOR Mean} \\
        \midrule 
        Bottom 50\% link pred scores & 0.587982 & 0.555729 \\
        Top 50\% link pred scores & 0.605974 & 0.569973 \\
        \bottomrule 
    \end{tabular}
\end{table}

As shown in Table~\ref{tab:link_pred_confidence}, the generation quality of the low-confidence group remains comparable to that of the high-confidence group across ROUGE and METEOR metrics. This suggests that even when some retrieved neighbors originate from noisy edges, the reasoning module is able to filter, contextualize, and extract stylistically relevant information from the neighborhood. Overall, these results indicate that \method is robust to moderate imperfections in link prediction and can effectively leverage the noisy-but-useful relational signals present in sparse user–item graphs.

\subsection{\rebuttal{Personalization Sparsity Robustness}}
To further examine how \method behaves under different levels of personalization sparsity, we partition users in the test split by the number of real historical reviews available: users with 0 reviews (cold-start), 1 review, and 2+ reviews. This allows us to isolate how much \method depends on explicit user history versus the contextual and relational reasoning signals introduced by our framework.
As shown in Table~\ref{tab:sparsity_results}, \method demonstrates strong robustness across all sparsity levels and consistently outperforms the PGraph baseline. Notably, our model achieves meaningful improvements even in the cold-start setting. This behavior arises because, even when a user has no prior reviews, \method can still leverage contextual cues provided at inference time, including the review title, partial user-written text, or product description, to retrieve relevant neighbors and construct a personalized reasoning path. By contrast, prior personalization methods such as PGraph depend primarily on embedding-based retrieval, which is significantly less effective when a user lacks a profile or has only one review.

\begin{table}[h]
\centering
\scriptsize
\begin{tabular}{lcccc}
\toprule
\textbf{Sparsity Level} & \textbf{\method ROUGE-L} & \textbf{PGraph ROUGE-L} & \textbf{\method ROUGE-1} & \textbf{PGraph ROUGE-1} \\
\midrule
0 historical reviews (cold-start) & \textbf{0.160} & 0.125 & \textbf{0.204} & 0.183 \\
1 historical review               & \textbf{0.170} & 0.137 & \textbf{0.225} & 0.212 \\
2+ historical reviews             & \textbf{0.186} & 0.149 & \textbf{0.296} & 0.262 \\
\bottomrule
\end{tabular}
\caption{Performance under different sparsity levels of user history.}
\label{tab:sparsity_results}
\end{table}

\subsection{\rebuttal{Utility of Reasoning Path Selection}}
In this section, we further explore the utility of the reasoning path selection as introduced in Eq.~\ref{eq:reasoning_sample}. We analyze the ranked reasoning paths produced by the scoring metric $\Omega$. The distribution of candidate scores, shown in Table~\ref{tab:reasoning_rank_scores}, reveals clear separation among candidate paths, indicating that their quality varies and that the proposed selection mechanism is necessary for \method\ to identify the most coherent and stylistically aligned reasoning trace.
\begin{table}[h]
\centering
\small
\begin{tabular}{lccccc}
\toprule
\textbf{Rank} 
& 1 (lowest) 
& 2 
& 3 
& 4 
& 5 (highest) \\
\midrule
\textbf{Score} 
& 0.3943 
& 0.4465 
& 0.4709 
& 0.4964 
& 0.5137 \\
\bottomrule
\end{tabular}
\caption{Score distribution of ranked reasoning paths produced by~$\Omega$ (Eq.~8).}
\label{tab:reasoning_rank_scores}
\end{table}

These results indicate that while reasoning traces cannot be directly evaluated in isolation, the model benefits substantially from the ranked reasoning guidance.

\section{Theoretical Analysis of the Bias-Variance Trade-Off in \method}\label{appendices:theory}
As detailed in \cref{exp:k}, the hyperparameter $K$ determines the number of predicted items. In \cref{tab:k_sensitivity}, we demonstrate that \method, with the reasoning alignment, can more reliably utilize the additional retrieved context compared to other baselines with retrieval. We hypothesize the behavior corresponds to the bias-variance trade-off theory, where the reasoning serves as a regularization trick that can offset the trade-off and allow the variance reduction without bias increase. Note that in the below we use k instead of K to represent the number of predicted/synthetic.

\begin{proposition}[Bias--Variance trade-off]
\label{prop:tradeoff}
Let $\theta_u \in \mathbb{R}^d$ denote the user's latent style vector. 
We observe $n$ real samples 
$x_j=\theta_u+\varepsilon_j$ with $\mathbb{E}[\varepsilon_j]=0$, 
$\mathrm{Var}(\varepsilon_j)=\sigma^2 I$,
and $k$ synthetic samples 
$\tilde x_\ell=\theta_u+\Delta+\tilde\varepsilon_\ell$ 
with $\mathbb{E}[\tilde\varepsilon_\ell]=0$, 
$\mathrm{Var}(\tilde\varepsilon_\ell)=\tilde\sigma^2 I$,
where $\Delta \in \mathbb{R}^d$ is a fixed (unknown) bias.
Consider the pooled estimator
\[
\hat\theta_u \;=\; \frac{1}{n+k}\Big(\sum_{j=1}^n x_j + \sum_{\ell=1}^k \tilde x_\ell\Big).
\]
Then the (per-coordinate) mean squared error is
\[
\mathrm{MSE}(k)
\;=\;
\underbrace{\frac{n\sigma^2 + k\tilde\sigma^2}{(n+k)^2}}_{\text{variance}}
\;+\;
\underbrace{\Big(\frac{k}{n+k}\Big)^2 \|\Delta\|^2}_{\text{bias}^2/d \;\text{(per-dim)} }.
\]
In the equal-noise case $\tilde\sigma^2=\sigma^2$, this simplifies to
\[
\mathrm{MSE}(k)
\;=\;
\frac{\sigma^2}{n+k} \;+\; 
\Big(\frac{k}{n+k}\Big)^2 \|\Delta\|^2.
\]

\end{proposition}

\begin{proof}[Sketch]
$\mathbb{E}[\hat\theta_u]=\theta_u+\frac{k}{n+k}\Delta$, so the squared bias per dimension is $\big(\frac{k}{n+k}\big)^2\|\Delta\|^2$ (treating $\sigma^2$ as per-dimension noise). 
Since samples are independent with isotropic noise,
$\mathrm{Var}(\hat\theta_u)=\frac{n\sigma^2+k\tilde\sigma^2}{(n+k)^2}I$. 
Add variance and bias$^2$ to obtain the expression. 
For $\tilde\sigma^2=\sigma^2$, write $\mathrm{MSE}(t)=\frac{\sigma^2}{n}(1-t)+\|\Delta\|^2 t^2$, 
differentiate w.r.t.\ $t$, set to zero, and solve.
\end{proof}

\begin{remark}[Effect of Reasoning Alignment]
If reasoning alignment attenuates the preference mismatch to $\Delta_{\mathrm{RA}}=\beta\Delta$ with $\beta\in(0,1)$, then
\[
\mathrm{MSE}_{\mathrm{RA}}(k)
\;=\; \frac{\sigma^2}{n+k}
\;+\;
\Big(\frac{k}{n+k}\Big)^2 \beta^2 \|\Delta\|^2,
\]
so the minimum achievable error decreases and the optimal fraction 
$t^\star_{\mathrm{RA}}=\frac{\sigma^2}{2n\beta^2\|\Delta\|^2}$ increases, 
i.e., \emph{alignment lets you safely use larger $k$}.
\end{remark}

\begin{remark}[Sparse Users Benefit More]
$t^\star$ scales as $1/n$: when $n$ is small (sparse users), the variance term dominates and the optimal augmentation fraction is larger. 
Thus augmentation disproportionately helps sparse users by reducing variance, 
while reasoning alignment curbs the bias induced by $\Delta$.
\end{remark}

\section{\method Prompts}
\label{appendices:prompts}
In this section, we supply the prompts we used in \method. $\phi$ is used in \cref{eq:reasoning_sample} where the prompt is used to elicit candidate reasoning paths. $\xi$ is used in \cref{eq:reasoning_gold}, where the prompt is used to generate the final answer given the input and reasoning to evaluate the candidate paths. Lastly, $\rho$ is used in \cref{reasoning_final} where it structures the final input for personalized text generation.
\begin{mybox}{Reasoning Paths Generation ($\phi$)}
System: You are a personalized review generation assistant that generates high-quality reviews based on user history and context. \\\\
    Given profile which contains past documents written by the same person (might be empty), documents written by users that have similar writing style, reviews on the target product, and reasoning.
\\\\
User's own profile:
\{history\_reviews\_str\}

Similar profiles:
\{neighbor\_reviews\_str\}

Product Reviews:
\{product\_reviews\_str)\}
\\\\
Based on the above information, provide a detailed reasoning path that explains how we can arrive at the expected output. Consider:
\\
1. User's Writing Style: Analyze their typical review length, tone, and language patterns. \\
2. User's Preferences: What aspects of products do they typically focus on or value? \\ 
3. Product Information: What are the commonly mentioned features, pros, and cons from other reviews?
\\
Do not limit the reasoning to the above points. You can use your own knowledge to reason about the user's review. It is important to make sure that you only talk about information from the profile while considering the expected output in the reasoning process. You cannot directly copy or mention anything about the expected output. The expected output is only used to determine the reasoning process and how profile can affect the expected output.
\\\\
Provide your reasoning that leads to the following expected review on the target product from the user:

Expected Output:\\
Title: "{target\_review['title']}"\\
Text: "{target\_review['text']}"\\
Rating: {target\_review['rating']}\\
\\\\
As mentioned before, you cannot directly copy or mention anything about the expected output. The expected output is only used to determine the reasoning process. Do not mention the expected output in your reasoning. Your reasoning should only analyze the profile and the other reviews.
\\\\
Output your reasoning in a single paragraph. Do not output anything else. 
\\\\
Your reasoning:
\end{mybox}

\newpage 

\begin{mybox}{Reasoning Paths Evaluation ($\xi$)}
System: You are a personalized review evaluation assistant that judges whether the generated reasoning and review are consistent with the user's style and product context. \\\\
    Given a profile containing past documents written by the same person (may be empty), documents from users with similar writing style, reviews on the target product, and a reasoning trace, you will evaluate and refine the review text. 
\\\\
User's own profile:
\{history\_reviews\_str\}
\\\\
Similar profiles:
\{neighbor\_reviews\_str\}
\\\\
Product Reviews:
\{product\_reviews\_str)\}
\\\\
Reasoning:
\{reasoning\_str\}
\\\\
Based on the above information, evaluate how well the provided review text follows the reasoning and user profile. Consider:
\\
1. Faithfulness to the reasoning: Does the review follow the logical path outlined in the reasoning? \\
2. Stylistic alignment: Does the review reflect the user's writing style and preferences? \\
3. Product grounding: Is the review consistent with the product reviews and features mentioned?
\\\\
Do not copy directly from the reasoning or profiles. Your task is to provide a short evaluation and, if needed, produce a refined review text. 
\\\\
Provide your output strictly in the format: \\
Evaluation: <evaluation>. Review text: <Review text> 
\\\\
Do not output anything else. 
\\\\
Review text: \{review\_text\}
\end{mybox}

\begin{mybox}{Text Generation ($\rho$)}
System: You are a personalized review generation assistant that generates high-quality reviews based on user history and context. \\\\
    Given a profile containing past documents written by the same person (may be empty), documents written by users with similar writing style, and reviews on the target product. 
\\\\
User's own profile:
\{history\_reviews\_str\}
\\\\
Similar profiles:
\{neighbor\_reviews\_str\}
\\\\
Product Reviews:
\{product\_reviews\_str)\}
\\\\
Reason and generate a review title/review text based on the following review text/review title. Use the format: \\
Reasoning: <reasoning>. Review title/text: <Review title/text>. \\\\
Do not output anything else. 
\\\\
Review text/Review title: \{review\_text\} / \{review\_title\}
\end{mybox}
\newpage
\section{Notations}
To facilitate readability, we summarize the main mathematical symbols and notations used throughout the paper in Table~\ref{tab:notation}, which serves as a quick reference to clarify definitions of variables, functions, and operators appearing in the main text.
\begin{table}[h]
\centering
\scriptsize
\renewcommand{\arraystretch}{1.15}
\caption{Summary of key notations used throughout the paper.}
\label{tab:notation}
\begin{tabular}{@{}llp{4.8cm}@{}}
\toprule
\textbf{Symbol} & \textbf{Type / Shape} & \textbf{Description} \\
\midrule
$G=(\mathcal{U}\cup \mathcal{I},\mathcal{E})$ & graph & Graph with user set $\mathcal{U}$, items $\mathcal{I}$, edges $\mathcal{E}$ \\
$H_u$ & set of texts & Observed history (reviews) of user $u$ \\
$R_i$ & set of texts & Reviews associated with item $i$ \\
$\text{Enc}(\cdot)$ & $\text{text} \to \mathbb{R}^d$ & Text encoder (e.g. SentenceTransformers) \\
$h_v^{(\ell)}$ & $\mathbb{R}^d$ & Node representation at layer $\ell$ (GraphSAGE) \\
$z_v=h_v^{(L)}$ & $\mathbb{R}^d$ & Final node embedding for node $v$ \\
$\mathcal{N}(v)$ & set of nodes & Neighborhood of node $v$ \\
$\mathbf{AGG}_\ell(\cdot)$ & operator & Neighborhood aggregator at layer $\ell$ \\
$\mathcal{M}$ & LLM & Base language model \\
$\mathcal{M'}$ & LLM & Fine-tuned LLM used for inference \\
$s_{u,i}$ & $\mathbb{R}$ & Link score from decoder for user $u$ and item $i$ \\
$K$ & integer & \# predicted items to augment per user \\
$k_{\mathrm{sim}}$ & integer & \# similar users retrieved for $u$ \\
$k_{\mathrm{peer}}$ & integer & \# peer texts (BM25) per item \\
$\mathcal{S}_u$ & set of users & Top-$k_{\mathrm{sim}}$ similar users to $u$ \\
$P_{u,i}$ & set of texts & Top-$k_{\mathrm{peer}}$ peer reviews for item $i$ \\
$\mathcal{I}_u^K$ & set of items & Top-$K$ predicted items for user $u$ by $s(u,i)$ \\
$t_{u,i}$ & text & Ground-truth review by $u$ for item $i$ (observed) \\
$\tilde{t}_{u,i}$ & text & Synthetic review for $(u,i)$ during expansion \\
$\hat{t}_{u,i^\star}$ & text & Final predicted personalized review for user $u$ on target item $i^\star$ \\
$\mathcal{Z}$ & text & A reasoning path \\
$\{Z^{(k)}\}_{k=1}^K$ & list of texts & $K$ candidate reasoning paths \\
$\phi(x, t_{u, j})$ & prompt & Prompt to elicit candidate reasoning paths given input and expected output\\
$\xi(x,Z)$ & prompt & Prompt that conditions generation on input $x$ and rationale $\mathcal{Z}$ to evaluate $\mathcal{Z}$ \\
$\rho(x)$ & prompt & Prompt that instructs model to output the reasoning path and the review \\
$\Omega(\cdot,\cdot)$ & metric & Evaluation metric for reasoning paths (e.g., ROUGE/METEOR) \\
\bottomrule
\end{tabular}
\end{table}

\end{document}